\documentclass[letterpaper, 10 pt, journal, twoside]{support/IEEEtran}

\let\labelindent\relax

\usepackage{times}
\usepackage[pdftex]{graphicx}
\usepackage{subfigure}
\usepackage{amsmath,amssymb,amsopn,amstext,amsfonts}
\usepackage{cancel}
\usepackage[space]{cite}
\usepackage{soul}
\usepackage{cuted}
\usepackage{lipsum,multicol}
\usepackage{stfloats}
\everymath{\displaystyle}
\usepackage{booktabs}
\usepackage{threeparttable}
\usepackage{amsthm}
\usepackage{xfrac}
\usepackage{balance}
\usepackage{color}
\usepackage{mathtools}
\DeclareMathOperator*{\med}{med}
\usepackage{tabularx}
\usepackage{bm}
\usepackage{diagbox}
\usepackage{float}
\usepackage{epstopdf}
\usepackage{url}
\usepackage{multirow}
\usepackage{multicol}
\usepackage{tikz}
\usepackage[linkcolor=black,citecolor=black,urlcolor=black,colorlinks=true]{hyperref}
\usepackage{enumitem}
\usepackage[ruled,vlined,linesnumbered]{algorithm2e}
\usepackage{pifont}
\DeclareMathAlphabet\mathbfcal{OMS}{cmsy}{b}{n}

\newtheorem{lemma}{Lemma}
\renewcommand{\qedsymbol}{$\blacksquare$}

\soulregister\cite7
\soulregister\citep7
\soulregister\citet7
\soulregister\ref7
\soulregister\it7
\soulregister\pageref7

\usepackage{arydshln}
\graphicspath{{figures/}}
\DeclareGraphicsExtensions{.pdf,.png,.jpg,.eps}

\IEEEoverridecommandlockouts

\title{\LARGE \bf FAST-LIO2: Fast Direct LiDAR-inertial Odometry}

\author{Wei Xu$^{*1}$, Yixi Cai$^{*1}$, Dongjiao He$^{1}$, Jiarong Lin$^{1}$, Fu Zhang$^{1}$
    \thanks{$^{*}$These two authors contribute equally to this work.}
	\thanks{$^{1}$All authors are with Department of Mechanical Engineering, University of Hong Kong. {\it \{xuweii,yixicai,hdj65822,jiarong.lin\}}{\it @connect.hku.hk, fuzhang@hku.hk}}
}%

\begin{document}
\maketitle
\begin{abstract}
This paper presents FAST-LIO2: a fast, robust, and versatile LiDAR-inertial odometry framework. Building on a highly efficient tightly-coupled iterated Kalman filter, FAST-LIO2 has two key novelties that allow fast, robust, and accurate LiDAR navigation (and mapping). The first one is directly registering raw points to the map (and subsequently update the map, i.e., mapping) without extracting features. This enables the exploitation of subtle features in the environment and hence increases the accuracy. The elimination of a hand-engineered feature extraction module also makes it naturally adaptable to emerging LiDARs of different scanning patterns; The second main novelty is maintaining a map by an incremental k-d tree data structure, \textit{ikd-Tree}, that enables incremental updates (i.e., point insertion, delete) and dynamic re-balancing. Compared with existing dynamic data structures (octree, R$^\ast$-tree, \textit{nanoflann} k-d tree), \textit{ikd-Tree} achieves superior overall performance while naturally supports downsampling on the tree. We conduct an exhaustive benchmark comparison in 19 sequences from a variety of open LiDAR datasets. FAST-LIO2 achieves consistently higher accuracy at a much lower computation load than other state-of-the-art LiDAR-inertial navigation systems. Various real-world experiments on solid-state LiDARs with small FoV are also conducted. Overall, FAST-LIO2 is computationally-efficient (e.g., up to 100 $Hz$ odometry and mapping in large outdoor environments), robust (e.g., reliable pose estimation in cluttered indoor environments with rotation up to 1000 $deg/s$), versatile (i.e., applicable to both multi-line spinning and solid-state LiDARs, UAV and handheld platforms, and Intel and ARM-based processors), while still achieving higher accuracy than existing methods. Our implementation of the system FAST-LIO2, and the data structure \textit{ikd-Tree} are both open-sourced on Github\footnote[2]{\label{ftnt:fastlio}\url{https://github.com/hku-mars/FAST_LIO}}$^{\textstyle ,}$\footnote[3]{\label{ftnt:ikdtree}\url{https://github.com/hku-mars/ikd-Tree}}.
\end{abstract}

\section{Introduction}
Building a dense 3-dimension (3D) map of an unknown environment in real-time and simultaneously localizing in the map (i.e., SLAM) is crucial for autonomous robots to navigate in the unknown environment safely. The localization provides state feedback for the robot onboard controllers, while the dense 3D map provides necessary information about the environment (i.e., free space and obstacles) for trajectory planning. Vision-based SLAM \cite{forster2016svo, forster2016manifold, qin2018vins, campos2021orb} is very accurate in localization but maintains only a sparse feature map and suffers from illumination variation and severe motion blur. On the other hand, real-time dense mapping \cite{newcombe2012dense, meilland20133d, bloesch2018codeslam, kerl2013dense} based on visual sensors at high resolution and accuracy with only the robot onboard computation resources is still a grand challenge.

Due to the ability to provide direct, dense, active, and accurate depth measurements of environments, 3D light detection and ranging (LiDAR) sensor has emerged as another essential sensor for robots \cite{thrun2006stanley, urmson2008autonomous}. Over the last decade, LiDARs have been playing an increasingly important role in many autonomous robots, such as self-driving cars \cite{levinson2011towards} and autonomous UAVs \cite{liu2017planning, gao2019flying}. Recent developments in LiDAR technologies have enabled the commercialization and mass production of more lightweight, cost-effective (in a cost range similar to global shutter cameras), and high performance (centimeter accuracy at hundreds of meters measuring range) solid-state LiDARs~\cite{wang2020mems, liu2020low}, drawing much recent research interests~\cite{lin2019loam_livox,  li2021towards,lin2019fast,wang2021lightweight,liu2021balm}. The considerably reduced cost, size, weight, and power of these LiDARs hold the potential to benefit a broad scope of existing and emerging robotic applications.

40The central requirement for adopting LiDAR-based SLAM approaches to these widespread applications is to obtain accurate, low-latency state estimation and dense 3D map with limited onboard computation resources. However, efficient and accurate LiDAR odometry and mapping are still challenging problems: 1) Current LiDAR sensors produce a large amount of 3D points from hundreds of thousands to millions per second. Processing such a large amount of data in real-time and on limited onboard computing resources requires a high computation efficiency of the LiDAR odometry methods; 2) To reduce the computation load, features points, such as edge points or plane points, are usually extracted based on local smoothness. However, the performance of the feature extraction module is easily influenced by the environment. For example, in structure-less environments without large planes or long edges, the feature extraction will lead to few feature points. This situation is considerably worsened if the LiDAR Field of View (FoV) is small, a typical phenomenon of emerging solid-state LiDARs \cite{lin2019loam_livox}. Furthermore, the feature extraction also varies from LiDAR to LiDAR, depending on the scanning pattern (e.g., spinning, prism-based \cite{liu2020low}, MEMS-based \cite{wang2020mems}) and point density. So the adoption of a LiDAR odometry method usually requires much hand-engineering work; 3) LiDAR points are usually sampled sequentially while the sensor undergoes continuous motion. This procedure creates significant motion distortion influencing the performance of the odometry and mapping, especially when the motion is severe. Inertial measurement units (IMUs) could mitigate this problem but introduces additional states (e.g., bias, extrinsic) to estimate; 4) LiDAR usually has a long measuring range (e.g., hundreds of meters) but with quite low resolution between scanning lines in a scan. The resultant point cloud measurements are sparsely distributed in a large 3D space, necessitating a large and dense map to register these sparse points. Moreover, the map needs to support efficient inquiry for correspondence search while being updated in real-time incorporating new measurements. Maintaining such a map is a very challenging task and very different from visual measurements, where an image measurement is of high resolution, so requiring only a sparse feature map because a feature point in the map can always find correspondence as long as it falls in the FoV.

In this work, we address these issues by two key novel techniques: incremental k-d tree and direct points registration. More specifically, our contributions are as follows: 1) We develop an incremental k-d tree data structure, {\it ikd-Tree}, to represent a large dense point cloud map efficiently. Besides efficient nearest neighbor search, the new data structure supports incremental map update (i.e., point insertion, on-tree downsampling, points delete) and dynamic re-balancing at minimal computation cost. These features make the {\it ikd-Tree} very suitable for LiDAR odometry and mapping application, leading to 100 $Hz$ odometry and mapping on computationally-constrained platforms such as an Intel i7-based micro-UAV onboard computer and even ARM-based processors. The \textit{ikd-Tree} data structure toolbox is open-sourced on Github\footnotemark[3]. 2) Allowed by the increased computation efficiency of \textit{ikd-Tree}, we directly register raw points to the map, which enables more accurate and reliable scan registration even with aggressive motion and in very cluttered environments. We term this raw points-based registration as {\it direct method} in analogy to visual SLAM \cite{forster2014svo}. The elimination of a hand-engineered feature extraction makes the system naturally applicable to different LiDAR sensors; 3) We integrate these two key techniques into a full tightly-coupled lidar-inertial odometry system FAST-LIO \cite{xu2020fastlio} we recently developed. The system uses an IMU to compensate each point's motion via a rigorous back-propagation step and estimates the system's full state via an on-manifold iterated Kalman filter. To further speed up the computation, a new and mathematically equivalent formula of computing the Kalman gain is used to reduce the computation complexity to the dimension of the state (as opposed to measurements). The new system is termed as FAST-LIO2 and is open-sourced at Github\footnotemark[2] to benefit the community; 4) We conduct various experiments to evaluate the effectiveness of the developed {\it ikd-Tree}, the direct point registration, and the overall system. Experiments on 18 sequences of various sizes show that \textit{ikd-Tree} achieves superior performance against existing dynamic data structures (octree, R$^\ast$-tree, \textit{nanoflann} k-d tree) in the application of LiDAR odometry and mapping. Exhaustive benchmark comparison on 19 sequences from various open LiDAR datasets shows that FAST-LIO2 achieves consistently higher accuracy at a much lower computation load than other state-of-the-art LiDAR-inertial navigation systems. We finally show the effectiveness of FAST-LIO2 on challenging real-world data collected by emerging solid-state LiDARs with very small FoV, including aggressive motion (e.g., rotation speed up to 1000 $deg/s$) and structure-less environments.

The remaining paper is organized as follows: In Section. \ref{sec:related_work}, we discuss relevant research works. We give an overview of the complete system pipeline and the details of each key components in Section. \ref{sec:overview}, \ref{sec:stat_est} and \ref{sec:mapping}, respectively. The benchmark comparison on open datasets are presented in Section. \ref{sec:benchmark} and the real-world experiments are reported in Section. \ref{sec:experiment}, followed by conclusions in Section. \ref{sec:conclusion}.

\section{Related Works}\label{sec:related_work}
\subsection{LiDAR(-Inertial) Odometry}
Existing works on 3D LiDAR SLAM typically inherit the LOAM structure proposed in \cite{zhang2014loam}. It consists of three main modules: feature extraction, {\it odometry}, and {\it mapping}. In order to reduce the computation load, a new LiDAR scan first goes through feature points (i.e., edge and plane) extraction based on the local smoothness. Then the {\it odometry} module (scan-to-scan) matches feature points from two consecutive scans to obtain a rough yet real-time (e.g., $10Hz$) LiDAR pose odometry. With the odometry, multiple scans are combined into a sweep which is then registered and merged to a global map (i.e., {\it mapping}). In this process, the map points are used to build a k-d tree which enables a very efficient $k$-nearest neighbor search ($k$NN search). Then, the point cloud registration is achieved by the Iterative Closest Point (ICP)~\cite{sharp2002icp,low2004linear,segal2009generalizedicp} method wherein each iteration, several closest points in the map form a plane or edge where a target point belongs to. In order to lower the time for k-d tree building, the map points are downsampled at a prescribed resolution. The optimized mapping process is typically performed at a much low rate (1-2$Hz$).

Subsequent LiDAR odometry works keep a framework similar to LOAM. For example, Lego-LOAM~\cite{shan2018lego} introduces a ground point segmentation to decrease the computation and a loop closure module to reduce the long-term drift. Furthermore, LOAM-Livox~\cite{lin2019loam_livox} adopts the LOAM to an emerging solid-state LiDAR. In order to deal with the small FoV and non-repetitive scanning, where the features points from two consecutive scans have very few correspondences, the {\it odometry} of LOAM-Livox is obtained by directly registering a new scan to the global map. Such a direct scan to map registration increases odometry accuracy at the cost of increased computation for building a k-d tree of the updated map points at every step.

Incorporating an IMU can considerably increase the accuracy and robustness of LiDAR odometry by providing a good initial pose required by ICP. Moreover, the high-rate IMU measurements can effectively compensate for the motion distortion in a LiDAR scan. LION~\cite{tagliabue2021lion} is a loosely-coupled LiDAR inertial SLAM method that keeps the scan-to-scan registration of LOAM and introduces an observability awareness check into the {\it odometry} to lower the point number and hence save the computation. More tightly-coupled LiDAR-inertial fusion works \cite{li2021towards, ye2019tightly, T2020liosam, qin2019lins} perform {\it odometry} in a small size local map consisting of a fixed number of recent LiDAR scans (or keyframes). Compared to scan-to-scan registration, the scan to local map registration is usually more accurate by using more recent information. More specifically, LIOM~\cite{ye2019tightly} presents a tightly-coupled LiDAR inertial fusion method where the IMU preintegrations are introduced into the {\it odometry}. LILI-OM \cite{li2021towards} develops a new feature extraction method for non-repetitive scanning LiDAR and performs scan registration in a small map consisting of 20 recent LiDAR scans for the {\it odometry}. The {\it odometry} of LIO-SAM~\cite{T2020liosam} requires a 9-axis IMU to produce attitude measurement as the prior of scan registration within a small local map. LINS~\cite{qin2019lins} introduces a tightly-coupled iterated Kalman filter and robocentric formula into the LiDAR pose optimization in the {\it odometry}. Since the local map in the above works is usually small to obtain real-time performance, the odometry drifts quickly, necessitating a low-rate {\it mapping} process, such as map refining (LINS~\cite{qin2019lins}), sliding window joint optimization (LILI-OM\cite{li2021towards} and LIOM~\cite{ye2019tightly}) and factor graph smoothing\cite{kaess2012isam2} (LIO-SAM~\cite{T2020liosam}). Compared to the above methods, FAST-LIO~\cite{xu2020fastlio} introduces a formal back-propagation that precisely considers the sampling time of every single point in a scan and compensates the motion distortion via a rigorous kinematic model driven by IMU measurements. Furthermore, a new Kalman gain formula is used to reduce the computation complexity from the dimension of the measurements to the dimension of the state. The new formula is proved to be mathematically equivalent to the conventional one but reduces the computation by several orders of magnitude. The considerably increased computation efficiency allows a direct and real-time scan to map registration in {\it odometry} and update the map (i.e., {\it mapping}) at every step. The timely mapping of all recent scan points leads to increased odometry accuracy. However, to prevent the growing time of building a k-d tree of the map, the system can only work in small environments (e.g., hundreds of meters). 

FAST-LIO2 builds on FAST-LIO \cite{xu2020fastlio} hence inheriting the tightly-coupled fusion framework, especially the back-propagation resolving motion distortion and fast Kalman gain computation boosting the efficiency. To systematically address the growing computation issue, we propose a new data structure {\it ikd-Tree} which supports incremental map update at every step and efficient $k$NN inquiries. Benefiting from the drastically decreased computation load, the {\it odometry} is performed by directly registering raw LiDAR points to the map, such that it improves accuracy and robustness of odometry and mapping, especially when a new scan contains no prominent features (e.g., due to small FoV and/or structure-less environments). Compared to the above tightly-coupled LiDAR-inertial methods, which all use feature points, our method is more lightweight and achieves increased mapping rate and odometry accuracy, and eliminates the need for parameter tuning for feature extraction. 

The idea of directly registering raw points in our work has been explored in LION \cite{tagliabue2021lion}, which is however a loosely-coupled method as reviewed above. This idea is also very similar to the generalized-ICP (G-ICP) proposed in~\cite{segal2009generalizedicp}, where a point is registered to a small local plane in the map. This ultimately assumes that the environment is smooth and hence can be viewed as a plane locally. However, the computation load of generalized-ICP is usually large~\cite{koide2020voxelized}. Other works based on Normal Distribution Transformation (NDT)~\cite{biber2003normal,magnusson2009three,magnusson2009evaluation} also register raw points, but NDT has lower stability compared to ICP and may diverge in some scenes~\cite{magnusson2009evaluation}.

\subsection{Dynamic Data Structure in Mapping}
In order to achieve real-time mapping, a dynamic data structure is required to support both incremental updates and $k$NN search with high efficiency. Generally, the $k$NN search problem can be solved by building spatial indices for data points, which can be divided into two categories: partitioning the data and splitting the space. A well-known instance to partition the data is R-tree\cite{guttman1984rtree} which clusters the data into potential overlapped axis-aligned cuboids based on data proximity in space. Various R-trees splits the nodes by linear, quadratic, and exponential complexities, all supporting nearest neighbor search and point-wise updating (insertion, delete, and re-insertion). Furthermore, R-trees also support searching target data points in a given search area or satisfying a given condition. Another version of R-trees is R$^\ast$-tree which outperforms the original ones\cite{beckmann1990rstar}. The R$^\ast$-tree handles insertion by minimum overlap criteria and applies a forced re-insertion principle for the node splitting algorithm. 

Octree \cite{meagher1982octree} and k-dimensional tree (k-d tree)\cite{bentley1975kdtree} are two well-known types of data structures to split the space for $k$NN search. The octree organizes 3-D point clouds by splitting the space equally into eight axis-aligned cubes recursively. The subdivision of a cube stops when the cube is empty, or a stopping rule (e.g., minimal resolution or minimal point number) is met. New points are inserted to leaf nodes on the octree while a further subdivision is applied if necessary. The octree supports both $k$NN search and box-wise search, which returns data points in a given axis-aligned cuboid. 

The k-d tree is a binary tree whose nodes represent an axis-aligned hyperplane to split the space into two parts. In the standard construction rule, the splitting node is chosen as the median point along the longest dimension to achieve a compact space division\cite{friedman1977search}. When considering the data characteristics of low dimensionality and storage on main memory in mapping, comparative studies show that k-d trees achieve the best performance in $k$NN problem\cite{vermeulen2017comparative,libnabo2012}. However, inserting new points to and deleting old points from a k-d tree deteriorates the tree's balance property; thus, re-building is required to re-balance the tree. Mapping methods using k-d tree libraries, such as ANN\cite{arya1998ann}, \textit{libnabo}\cite{libnabo2012} and FLANN\cite{muja2009flannfast}, fully re-build the k-d trees to update the map, which results in considerable computation. Though hardware-based methods to re-build k-d trees have been thoroughly investigated in 3D graphic applications\cite{huntsinglecore,popov2006sah,shevtsovmulticore,zhougpu}, the proposed methods rely heavily on the computational sources which are usually limited on onboard computers for robotic applications. Instead of re-building the tree in full scale, Galperin \textit{et al.} proposed a scapegoat k-d tree where re-building is applied partially on the unbalanced sub-trees to maintain a loose balance property of the entire tree\cite{scapegoat}. Another approach to enable incremental operations is maintaining a set of k-d trees in a logarithmic method similar to \cite{bentley1980decomposable,overmars1987design} and re-building a carefully chosen sub-set. The Bkd-tree maintains a k-d tree $\mathcal{T}_0$ with maximal size $M$ in the main memory and a set of k-d trees $\mathcal{T}_i$ on the external memory where the $i$-th tree has a size of $2^{(i-1)}M$\cite{procopiuc2003bkd}. When the tree $\mathcal{T}_0$ is full, the points are extracted from $\mathcal{T}_0$ to $\mathcal{T}_{k-1}$ and inserted into the first empty tree $\mathcal{T}_k$. The state-of-the-art implementation \textit{nanoflann} k-d tree leverages the logarithmic structure for incremental updates, whereas lazy labels only mark the deleted points without removing them from the trees (hence memory) \cite{blanco2014nanoflann}. 

We propose a dynamic data structure based on the scapegoat k-d tree\cite{scapegoat}, named incremental k-d tree (\textit{ikd-Tree}), to achieve real-time mapping. Our \textit{ikd-Tree} supports point-wise insertion with on-tree downsampling which is a common requirement in mapping, whereas downsampling must be done outside before inserting new points into other dynamic data structures\cite{meagher1982octree,beckmann1990rstar,blanco2014nanoflann}. When it is required to remove unnecessary points in a given area with regular shapes (e.g., cuboids), the existing implementations of R-trees and octrees search the points within the given space and delete them one by one while common k-d trees use a radius search to obtain point indices. Compared to such an indirect and inefficient method, the \textit{ikd-Tree} deletes the points in given axis-aligned cuboids directly by maintaining range information and lazy labels. Points labeled as ``deleted'' are removed during the re-building process. Furthermore, though incremental updates are available after applying the partial re-balancing methods as the scapegoat k-d tree\cite{scapegoat} and \textit{nanoflann} k-d tree\cite{blanco2014nanoflann}, the mapping methods using k-d trees suffers from intermittent delay when re-building on a large number of points. In order to overcome this, the significant delay in \textit{ikd-Tree} is avoided by parallel re-building while the real-time ability and accuracy in the main thread are guaranteed.

\section{System Overview}\label{sec:overview}
\begin{figure*}[t]
    \begin{center}
        {\includegraphics[width=2.05\columnwidth]{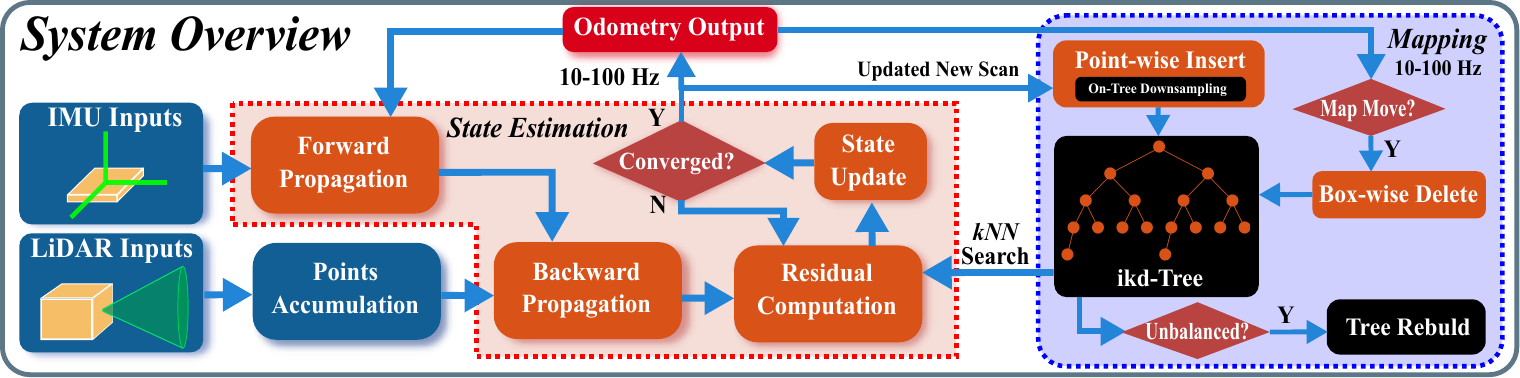}}
    \end{center}
    \vspace{-0.4cm}
    \caption{\label{fig:workflow}System overview of FAST-LIO2.}
    \vspace{-0.2cm}
\end{figure*}
The pipeline of FAST-LIO2 is shown in Fig.~\ref{fig:workflow}. The sequentially sampled LiDAR raw points are first accumulated over a period between $10ms$ (for $100Hz$ update) and $100ms$ (for $10Hz$ update). The accumulated point cloud is called a scan. In order to perform state estimation, points in a new scan are registered to map points (i.e., {\it odometry}) maintained in a large local map via a tightly-coupled iterated Kalman filter framework (big dashed block in red, see Section. \ref{sec:stat_est}). Global map points in the large local map are organized by an incremental k-d tree structure {\it ikd-Tree} (big dashed block in blue, see Section. \ref{sec:mapping}). If the FoV range of current LiDAR crosses the map border, the historical points in the furthest map area to the LiDAR pose will be deleted from {\it ikd-Tree}. As a result, the {\it ikd-Tree} tracks all map points in a large cube area with a certain length (referred to as ``map size" in this paper) and is used to compute the residual in the state estimation module. The optimized pose finally registers points in the new scan to the global frame and merges them into the map by inserting to the {\it ikd-Tree} at the rate of odometry (i.e., {\it mapping}).

\section{State Estimation}\label{sec:stat_est}
The state estimation of FAST-LIO2 is a tightly-coupled iterated Kalman filter inherited from FAST-LIO\cite{xu2020fastlio} but further incorporates the online calibration of LiDAR-IMU extrinsic parameters. Here we briefly explain the essential formulations and workflow of the filter and refer readers to \cite{xu2020fastlio} for more details.

\subsection{Kinematic Model}
We first derive the system model, which consists of a state transition model and a measurement model. 
\subsubsection{State Transition Model}
\
\par
Take the first IMU frame (denoted as $I$) as the global frame (denoted as $G$) and denote $^{I}\mathbf T_{L}=\left(^I\mathbf R_{L},\;^I\mathbf p_L\right)$ the unknown extrinsic between LiDAR and IMU, the kinematic model is:
\begin{equation}~\label{e:kine_model_s}
\begin{aligned}
^G\dot{\mathbf R}_I&={}^G{\mathbf R}_I \lfloor \bm \omega_m -  \mathbf b_{\bm \omega} - \mathbf n_{\bm \omega} \rfloor_\wedge,\ ^G\dot{\mathbf p}_I={}^G\mathbf v_I,\\
^G\dot{\mathbf v}_I&={}^G{\mathbf R}_I \left( \mathbf a_m - \mathbf b_{\mathbf a} - \mathbf n_{\mathbf a} \right) + {}^G\mathbf g\\
\dot{\mathbf b}_{\bm \omega} &=\mathbf n_{\mathbf b\bm \omega}, \ \dot{\mathbf b}_{\mathbf a}=\mathbf n_{\mathbf b\mathbf a}, \\
^G\dot{\mathbf g} &= \mathbf 0,\ ^{I}\dot{\mathbf R}_{L}=\mathbf 0, \ ^{I}\dot{\mathbf p}_{L}=\mathbf 0
\end{aligned}
\end{equation}
where $^G\mathbf p_I$, $^G{\mathbf R}_I$ denote the IMU position and attitude in the global frame, $^G\mathbf g$ is the gravity vector in the global frame, $\mathbf a_m$ and $\bm \omega_m$ are IMU measurements, $\mathbf n_{\mathbf a}$ and $\mathbf n_{\bm \omega}$ denote the measurement noise of $\mathbf a_m$ and $\bm \omega_m$, $\mathbf b_{\mathbf a}$ and $\mathbf b_{\boldsymbol{\omega}}$ are the IMU biases modeled as random walk process driven by $\mathbf n_{\mathbf b \mathbf a}$ and $\mathbf n_{\mathbf b\bm \omega}$, and the notation $\lfloor \mathbf a \rfloor_{\wedge}$ denotes the skew-symmetric cross product matrix of vector $\mathbf a \in \mathbb{R}^3$.

Denote $i$ the index of IMU measurements. Based on the $\boxplus$ operation defined in~\cite{xu2020fastlio}, the continuous kinematic model (\ref{e:kine_model_s}) can be discretized at the IMU sampling period $\Delta t$\cite{he2021embedding}:
\begin{equation}\label{e:kine_model_discrete}
    \begin{aligned}
    \mathbf x_{i+1} &= \mathbf x_{i} \boxplus \left( \Delta t \mathbf f(\mathbf x_i, \mathbf u_i, \mathbf w_i)   \right)
    \end{aligned}
\end{equation}
where the function $\mathbf f$, state $\mathbf x$, input $\mathbf u$ and noise $\mathbf w$ are defined as below:
\begin{equation}
\begin{aligned}\notag
    \mathcal{M} &\triangleq SO(3) \times \mathbb{R}^{15} \times SO(3) \times \mathbb{R}^{3};\ \text{dim}(\mathcal{M}) = 24 \\
    \mathbf x &\triangleq \!\begin{bmatrix}^G{\mathbf R}_I^T\!&\!^G\mathbf p_I^T\!\!&\!^G\mathbf v_I^T\!&\!\mathbf b_{\bm \omega}^T\!&\!\!\mathbf b_{\mathbf a}^T\!\!&\!\!^G\mathbf g^T\!&\!\!^{I}\mathbf R_{L}^T \!&\!^{I}\mathbf p_{L}^T \!\end{bmatrix}^T\!\in\!\! \mathcal M\\
    \mathbf u &\triangleq \begin{bmatrix} {\bm \omega}^T_m & {\mathbf a}^T_m\end{bmatrix}^T,\
    \mathbf w \triangleq \begin{bmatrix}\mathbf n_{\bm \omega}^T&\mathbf n_{\mathbf a}^T&\mathbf n_{\mathbf b\bm \omega}^T&\mathbf n_{\mathbf b\mathbf a}^T\end{bmatrix}^T \\
    \mathbf f&\left( \mathbf x, \mathbf u, \mathbf w\right)\! =\!\! \begin{bmatrix} \bm \omega_{m} - \mathbf b_{\bm \omega} - \mathbf n_{\bm \omega} \\
    \!{}^G\mathbf v_{I}\!+\!\displaystyle{\frac{1}{2}}\!\left( {}^G{\mathbf R}_{I}\left(\mathbf a_{m}\!-\!\mathbf b_{\mathbf a}\!-\! \mathbf n_{\mathbf a}\right)\! +\!\! {}^G\mathbf g\right)\!\Delta t \\
    {}^G{\mathbf R}_{I}\left(\mathbf a_{m}\!-\!\mathbf b_{\mathbf a}\!-\! \mathbf n_{\mathbf a}\right)\! +\!\! {}^G\mathbf g  \\
    \mathbf n_{\mathbf b \bm \omega} \\
    \mathbf n_{\mathbf b \mathbf a} \\
    \mathbf 0_{3\times1} \\ 
    \mathbf 0_{3\times1} \\
    \mathbf 0_{3\times1} \end{bmatrix}\!\!\in\! \mathbb{R}^{24}
\end{aligned}
\end{equation}

\subsubsection{Measurement Model}
LiDAR typically samples points one after another. The resultant points are therefore sampled at different poses when the LiDAR undergoes continuous motion.  To correct this in-scan motion, we employ the back-propagation proposed in \cite{xu2020fastlio}, which estimates the LiDAR pose of each point in the scan with respect to the pose at the scan end time based on IMU measurements. The estimated relative pose enables us to project all points to the scan end-time based on the exact sampling time of each individual point in the scan. As a result, points in the scan can be viewed as all sampled simultaneously at the scan end-time. 

Denote $k$ the index of LiDAR scans and $\{{}^{L}\mathbf p_j, j = 1, \cdots, m\}$ the points in the $k$-th scan which are sampled at the local LiDAR coordinate frame $L$ at the scan end-time. Due to the LiDAR measurement noise, each measured point ${}^{L}\mathbf p_j$ is typically contaminated by a noise ${}^{L} \mathbf n_{j}$ consisting of the ranging and beam-directing noise. Removing this noise leads to the true point location in the local LiDAR coordinate frame ${}^{L} \mathbf p_{j}^{\text{gt}}$:

\begin{equation}\label{e:gt_meas}
    {}^{L} \mathbf p_{j}^{\text{gt}} = {}^{L} \mathbf p_{j} + {}^{L} \mathbf n_{j}.
\end{equation}

\begin{figure}[t]
    \begin{center}
        {\includegraphics[width=0.9\columnwidth]{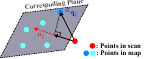}}
    \end{center}
    \caption{\label{fig:meas_model}The measurement model.}
\end{figure}

This true point, after projecting to the global frame using the corresponding LiDAR pose ${}^G{\mathbf T}_{I_k}=\left(^G\mathbf R_{I_k},\;^G\mathbf p_{I_k}\right)$ and extrinsic ${}^I{\mathbf T}_{L}$, should lie exactly on a local small plane patch in the map, i.e.,
\begin{equation}
\begin{split}~\label{e:meas_model}
    \mathbf 0 = {}^G \mathbf u_j^T \left( {}^G {\mathbf T}_{I_k} {}^{I}{\mathbf T}_{L} \left( {}^{L}{\mathbf p}_{j} + {}^{L}{\mathbf n}_{j} \right) - {}^G\mathbf q_j \right)
\end{split}
\end{equation}
where ${}^G \mathbf u_j$ is the normal vector of the corresponding plane and ${}^G\mathbf q_j$ is a point lying on the plane (see Fig. \ref{fig:meas_model}). It should be noted that the ${}^G{\mathbf T}_{I_k}$ and ${}^I{\mathbf T}_{L_k}$ are all contained in the state vector $\mathbf x_k$. The measurement contributed by the $j$-th point measurement ${}^{L}{\mathbf p}_{j}$ can therefore be summarized from (\ref{e:meas_model}) to a more compact form as below:
\begin{equation}
\begin{split}~\label{e:meas_equ}
    \mathbf 0 = \mathbf h_j\!\left({\mathbf x}_k, {}^{L}{\mathbf p}_{j} + {}^{L}{\mathbf n}_j \right),
\end{split}
\end{equation}
which defines an implicit measurement model for the state vector $\mathbf x_k$. 
\subsection{Iterated Kalman Filter}
Based on the state model (\ref{e:kine_model_discrete}) and measurement model (\ref{e:meas_equ}) formulated on manifold $\mathcal{M} \triangleq SO(3) \times \mathbb{R}^{15} \times SO(3) \times \mathbb{R}^{3} $, we employ an iterated Kalman filter directly operating on the manifold $\mathcal{M}$ following the procedures in \cite{he2021embedding} and \cite{xu2020fastlio}. It consists of two key steps: propagation upon each IMU measurement and iterated update upon each LiDAR scan, both step estimates the state naturally on the manifold $\mathcal{M}$ thus avoiding any re-normalization. Since the IMU measurements are typically at a higher frequency than a LiDAR scan (e.g., $200Hz$ for IMU measurement and $10Hz \sim 100Hz$ for LiDAR scans), multiple propagation steps are usually performed before an update. 

\subsubsection{Propagation}\label{sec:forward_prog}
Assume the optimal state estimate after fusing the last (i.e., $k-1$-th) LiDAR scan is $\bar{\mathbf {x}}_{k-1}$ with covariance matrix $\bar{\mathbf P}_{k-1}$. The forward propagation is performed upon the arrival of an IMU measurement. More specifically, the state and covariance are propagated following (\ref{e:kine_model_discrete}) by setting the process noise $\mathbf w_i$ to zero:
\begin{equation}\label{e:state_prop}
    \begin{aligned}
    \widehat{ \mathbf  x}_{i+1} &= \widehat{ \mathbf  x}_{i} \boxplus \left( \Delta t \mathbf f (\widehat{\mathbf x}_i, \mathbf u_i, \mathbf 0) \right); \ \widehat{\mathbf x}_{0} = \bar{\mathbf x}_{k-1},\\
    \widehat{\mathbf P}_{i+1} &= \mathbf F_{\widetilde{\mathbf x}_i}\widehat{\mathbf P}_{i}\mathbf F_{\widetilde{\mathbf x}_i}^T + \mathbf F_{\mathbf w_i} \mathbf Q_i \mathbf F_{\mathbf w_i}^T ; \ \widehat{\mathbf P}_{0} = \bar{\mathbf P}_{k-1},
    \end{aligned}
\end{equation}
where $\mathbf Q_i $ is the covariance of the noise $\mathbf w_i $ and the matrix $\mathbf F_{\widetilde{\mathbf x}_i}$ and $\mathbf F_{\mathbf w_i}$ are computed as below (see more abstract derivation in \cite{he2021embedding} and more concrete derivation in \cite{xu2020fastlio}):
\begin{equation}
\begin{split}~\label{e:noise_cov}
\mathbf F_{\widetilde{\mathbf x}_i} &=\textstyle \frac{\partial \left({\mathbf x}_{i+1} \boxminus \widehat{\mathbf x}_{i+1}\right)}{\partial \widetilde{\mathbf x}_i}|_{\begin{subarray}{l}\widetilde{\mathbf x}_i = \mathbf 0 ,\ \mathbf w_i = \mathbf 0 \end{subarray}} \\
\mathbf F_{\mathbf w_i} &=\left. \textstyle \frac{\partial \left({\mathbf x}_{i+1} \boxminus \widehat{\mathbf x}_{i+1}\right)}{\partial \mathbf w_i}\right|_{\begin{subarray}{l}\widetilde{\mathbf x}_i = \mathbf 0 ,\ \mathbf w_i = \mathbf 0 \end{subarray}} \\
\end{split}
\end{equation}

The forward propagation continues until reaching the end time of a new (i.e., $k$-th) scan where the propagated state and covariance are denoted as $\widehat{\mathbf{x}}_{k}, \widehat{\mathbf P}_{k}$.

\subsubsection{Residual Computation}\label{sec:residual}
Assume the estimate of state $\mathbf x_k$ at the current iterate update (see Section.  \ref{sec:iter-update}) is $\widehat{\mathbf x}_k^\kappa$, when $\kappa = 0$ (i.e., before the first iteration), $\widehat{\mathbf x}_k^\kappa = \widehat{\mathbf x}_k$, the predicted state from the propagation in (\ref{e:state_prop}). Then, we project each measured LiDAR point ${}^{L} \mathbf p_{j}$ to the global frame ${}^{G} \widehat{\mathbf p}_{j} = {}^G\widehat{\mathbf T}^\kappa_{I_k} \! {}^{I}\widehat{\mathbf T}^\kappa_{L_k}\!{}^{L}{\mathbf p}_{j} $ and search its nearest 5 points in the map represented by {\it ikd-Tree} (see Section. \ref{sec:map_manage}). The found nearest neighbouring points are then used to fit a local small plane patch with normal vector ${}^G  \mathbf u_j$ and centroid ${}^G  \mathbf q_j$ that were used in the measurement model (see (\ref{e:meas_model}) and (\ref{e:meas_equ})).  Moreover, approximating the measurement equation (\ref{e:meas_equ}) by its first order approximation made at $\widehat{\mathbf x}_k^\kappa$ leads to
\vspace{-0.1cm}
\begin{equation}\label{e:meas_jocob}
\begin{split}
    \mathbf 0 &= \mathbf h_j \left( \mathbf x_k, {}^{L}{\mathbf n}_{j} \right) \simeq \mathbf h_j \left( \widehat{\mathbf x}_k^\kappa, \mathbf 0 \right) + \mathbf H_j^\kappa  \widetilde{\mathbf x}_{k}^\kappa + \mathbf v_j\\&=  \mathbf z_j^\kappa + \mathbf H_j^\kappa  \widetilde{\mathbf x}_{k}^\kappa + \mathbf v_j
\end{split}
\end{equation}
where $\widetilde{\mathbf x}_{k}^\kappa=\mathbf x_k \boxminus \widehat{\mathbf x}_k^\kappa$ (or equivalently $\mathbf x_k = \widehat{\mathbf x}_k^\kappa \boxplus \widetilde{\mathbf x}_{k}^\kappa$), $\mathbf H_j^\kappa$ is the Jacobin matrix of {$\mathbf h_j\! \left(  \widehat{\mathbf x}_k^\kappa \boxplus \widetilde{\mathbf x}_{k}^\kappa , {}^{L}{\mathbf n}_{j} \right)$} with respect to $\widetilde{\mathbf x}_{k}^\kappa$, evaluated at zero, $\mathbf v_j \in \mathcal{N}(\mathbf 0, \mathbf R_j)$ is due to the raw measurement noise ${}^{L}{\mathbf n}_{j}$, and $\mathbf z_j^{\kappa}$ is called the residual:
\begin{equation}\label{e:res_compute}
\begin{split}
\mathbf z_j^\kappa = \mathbf h_j \left( \widehat{\mathbf x}_k^\kappa, \mathbf 0 \right) = \mathbf u^T_j \left({}^G\widehat{\mathbf T}^\kappa_{I_k} \! {}^{I}\widehat{\mathbf T}^\kappa_{L_k}\!{}^{L}{\mathbf p}_{j} \! - {}^G\mathbf q_j \right)
\end{split}
\end{equation}

\subsubsection{Iterated Update}\label{sec:iter-update}
The propagated state $\widehat{\mathbf x}_{k}$ and covariance $\widehat{\mathbf P}_{k}$ from Section. \ref{sec:forward_prog} impose a prior Gaussian distribution for the unknown state $\mathbf x_k$. More specifically, $\widehat{\mathbf P}_{k}$ represents the covariance of the following error state: 
\vspace{-0.1cm}
\begin{equation} \label{eq:prior}
\begin{aligned}
    \mathbf x_k \boxminus \widehat{\mathbf x}_k &= \left( \widehat{\mathbf x}_{k}^\kappa \boxplus \widetilde{\mathbf x}_{k}^\kappa \right) \boxminus \widehat{\mathbf x}_k = \widehat{\mathbf x}_{k}^\kappa \boxminus \widehat{\mathbf x}_{k} + \mathbf J^\kappa \widetilde{\mathbf x}_{k}^\kappa \\
    & \sim \mathcal{N}(\mathbf 0, \widehat{\mathbf P}_{k})
\end{aligned}
\end{equation}
where $\mathbf J^\kappa$ is the partial differentiation of $\left( \widehat{\mathbf x}_{k}^\kappa \boxplus \widetilde{\mathbf x}_{k}^\kappa \right) \boxminus \widehat{\mathbf x}_k$ with respect to $\widetilde{\mathbf x}_{k}^\kappa$ evaluated at zero: 
\vspace{-0.1cm}
\begin{equation}
\begin{split}~\label{e:L_compute}
    \mathbf J^\kappa \!\!=\!\!\begin{bmatrix}\begin{small}\mathbf A\!\left(\delta{}^G\!\bm{\theta}_{I_k}\right)\!\!^{-T} \!\!\end{small}& \mathbf 0_{3\times 15} & \mathbf 0_{3\times 3} & \mathbf 0_{3\times 3} \\
\mathbf 0_{15\times 3} & \mathbf I_{15\times 15} & \mathbf 0_{3\times 3} & \mathbf 0_{3\times 3}  \\
\mathbf 0_{3\times 3} & \mathbf 0_{3\times 15} & \mathbf A\!\left(\delta{}^I\!\bm{\theta}_{L_k}\right)\!\!^{-T} \!\!& \mathbf 0_{3\times 3}\\
\mathbf 0_{3\times 3} & \mathbf 0_{3\times 15} & \mathbf 0_{3\times 3} & \mathbf I_{3\times 3}  \\\end{bmatrix}\\
\end{split}
\end{equation}
where $\mathbf A(\cdot)^{-1}$ is defined in~\cite{he2021embedding, xu2020fastlio}, $\delta{}^G\!\bm{\theta}_{I_k} \!\!=\!\! ^G\widehat{\mathbf R}_{I_k}^\kappa \boxminus {}^G\widehat{\mathbf R}_{I_k}$ and $\delta{}^I\!\bm{\theta}_{L_k} = {}^I\widehat{\mathbf R}_{L_k}^\kappa\boxminus {}^I\widehat{\mathbf R}_{L_k}$ is the error states of IMU's attitude and rotational extrinsic, respectively. For the first iteration, $\widehat{\mathbf x}_{k}^\kappa \!=\! \widehat{\mathbf x}_{k}$, then $\mathbf J^\kappa \!=\! \mathbf I$.

Besides the prior distribution, we also have a distribution of the state due to the measurement (\ref{e:meas_jocob}):
\begin{equation}\label{eq:posteriori}
    -\mathbf v_j = \mathbf z_j^\kappa + \mathbf H_j^\kappa  \widetilde{\mathbf x}_{k}^\kappa \sim \mathcal{N}(\mathbf 0, \mathbf R_j)
\end{equation}

Combining the prior distribution in (\ref{eq:prior}) with the measurement model from (\ref{eq:posteriori}) yields the posteriori distribution of the state $\mathbf x_k$ equivalently represented by $\widetilde{\mathbf x}_{k}^\kappa$ and its maximum a-posteriori estimate (MAP):
\begin{equation}
\begin{split}~\label{e:error_states_solution}
    \min_{\widetilde{\mathbf x}_{k}^\kappa} \left( \| \mathbf x_k \boxminus \widehat{\mathbf x}_k \|^2_{ \widehat{\mathbf P}_k} + \sum\nolimits_{j=1}^{m} \| \mathbf z_j^\kappa + \mathbf H_j^\kappa \widetilde{\mathbf x}_{k}^\kappa \|^2_{\mathbf R_j} \right)
\end{split}
\end{equation}
where $\| \mathbf x \|_{\mathbf M}^2 = \mathbf x^T \mathbf M^{-1} \mathbf x$. This MAP problem can be solved by iterated Kalman filter as below (to simplify the notation, let $\mathbf H\! = \![ \mathbf H_1^{\kappa^T}, \cdots, \mathbf H_m^{\kappa^T}]^T$, $\mathbf R\! =\! \text{diag}\left(\mathbf R_1, \cdots \mathbf R_m \right)$,$\mathbf P\! =\!\left(\mathbf J^{\kappa}\right)^{-1} \widehat{\mathbf P}_{k} (\mathbf J^{\kappa})^{-T} $, and $\mathbf z_k^\kappa = \left[ \mathbf z_1^{\kappa^T}, \cdots, \mathbf z_m^{\kappa^T} \right]^T$):
\vspace{-0.1cm}
\begin{equation}
\begin{split}~\label{e:kalman_gain}
\mathbf K &= \left(\mathbf H^T {\mathbf R}^{-1} \mathbf H + {\mathbf P}^{-1} \right)^{-1}\mathbf H^T \mathbf R^{-1}, \\
\widehat{\mathbf x}_{k}^{\kappa+1} & = \widehat{\mathbf x}_{k}^{\kappa} \! \boxplus \!  \left( -\mathbf K  {\mathbf z}_k^\kappa  - (\mathbf I - \mathbf K \mathbf H ) (\mathbf J^\kappa)^{-1} \left( \widehat{\mathbf x}_{k}^{\kappa} \boxminus \widehat{\mathbf x}_{k} \right)  \right).
\end{split}
\end{equation}
Notice that the Kalman gain $\mathbf K$ computation needs to invert a matrix of the state dimension instead of the measurement dimension used in previous works. 

The above process repeats until convergence (i.e., $\| \widehat{\mathbf x}_{k}^{\kappa+1} \boxminus \widehat{\mathbf x}_{k}^\kappa\|\!<\! \epsilon$). After convergence, the optimal state and covariance estimates are:
\vspace{-0.1cm}
\begin{equation}
\begin{split}~\label{e:state_update}
\bar{\mathbf x}_{k} &= \widehat{\mathbf x}_{k}^{\kappa+1},\ \bar{\mathbf P}_k = \left( \mathbf I - \mathbf K \mathbf H \right) {\mathbf P}
\end{split}
\end{equation}

With the state update $\bar{\mathbf x}_k$, each LiDAR point (${}^{L} \mathbf p_{j}$) in the $k$-th scan is then transformed to the global frame via:
\begin{equation}
\begin{split}~\label{e:feat_global_map}
    {}^G \bar{\mathbf p}_{j} = {}^G \bar{\mathbf T}_{I_k} \! {}^{I}\bar{\mathbf T}_{L_k} {}^{L}{\mathbf p}_{j} ; \ j= 1, \cdots, m.
\end{split}
\end{equation}
The transformed LiDAR points $\{{}^G \bar{\mathbf p}_{j}\}$ are inserted to the map represented by {\it ikd-Tree} (see Section. \ref{sec:mapping}). Our state estimation is summarized in {\bf Algorithm \ref{alg:state_estimation}}.

\begin{algorithm}[t]
    \caption{State Estimation}
    \label{alg:state_estimation}
    \SetKwInOut{Input}{Input}\SetKwInOut{Output}{Output}\SetKwInOut{Start}{Start}\SetKwInOut{blank}{}
    \Input{Last output $\bar{\mathbf x}_{k-1}$ and $\bar{\mathbf P}_{k-1}$;\\
    LiDAR raw points in current scan;\\
    IMU inputs ($\mathbf a_{m}$, $\bm \omega_m$) during current scan.}
    Forward propagation to obtain state prediction $\widehat{\mathbf x}_k$ and its covariance $\widehat{\mathbf P}_k$ via (\ref{e:state_prop});\\
    Backward propagation to compensate motion \cite{xu2020fastlio};\\
    $\kappa = -1$, $\widehat{\mathbf x}_{k}^{\kappa = 0} = \widehat{\mathbf x}_{k}$; \\
        \Repeat{$\| \widehat{\mathbf x}_{k}^{\kappa+1} \boxminus \widehat{\mathbf x}_{k}^\kappa\|< \epsilon$}{
        $\kappa =\kappa+1$;\\
        Compute $\mathbf J^\kappa$ via (\ref{e:L_compute}) and $\mathbf P\! =\! (\mathbf J^{\kappa})^{-1} \widehat{\mathbf P}_{k} (\mathbf J^{\kappa})^{-T} $;\\
        Compute residual $\mathbf z_j^{\kappa}$ and Jacobin $\mathbf H_j^\kappa$ via (\ref{e:meas_jocob}) (\ref{e:res_compute});\\
        Compute the state update $\widehat{\mathbf x}_{k}^{\kappa+1}$ via (\ref{e:kalman_gain});\\}
    $\bar{\mathbf x}_{k} = \widehat{\mathbf x}_{k}^{\kappa+1}$; $\bar{\mathbf P}_{k} = \left( \mathbf I - \mathbf K \mathbf H \right) {\mathbf P}$; \\
    Obtain the transformed LiDAR points $\{{}^G \bar{\mathbf p}_{j}\}$ via (\ref{e:feat_global_map}). \\
    \Output{Current optimal estimate $\bar{\mathbf x}_{k}$ and $\bar{\mathbf P}_{k}$;\\
    The transformed LiDAR points $\{{}^G \bar{\mathbf p}_{j}\}$.}
\end{algorithm}

	\section{Mapping}\label{sec:mapping}
    	In this section, we describe how to incrementally maintain a map (i.e., insertion and delete) and perform $k$-nearest search on it by \textit{ikd-Tree}. In order to prove the time efficiency of \textit{ikd-Tree} theoretically, a complete analysis of time complexity is provided.
    	
    	\subsection{Map Management}\label{sec:map_manage}
    	The map points are organized into an \textit{ikd-Tree}, which dynamically grows by merging a new scan of point cloud at the odometry rate. To prevent the size of the map from going unbound, only map points in a large local region of length $L$ around the LiDAR current position are maintained on the \textit{ikd-Tree}. A 2D demonstration is shown in Fig. \ref{fig:localmap}. The map region is initialized as a cube with length $L$, which is centered at the initial LiDAR position $\mathbf{p}_0$. The detection area of LiDAR is assumed to be a detection ball centered at the LiDAR current position obtained from (\ref{e:state_update}). The radius of the detection ball is assumed to be $r = \gamma R$ where $R$ is the LiDAR FoV range, and $\gamma$ is a relaxation parameter larger than $1$. When the LiDAR moves to a new position $\mathbf{p'}$ where the detection ball touches the boundaries of the map, the map region is moved in a direction that increases the distance between the LiDAR detection area and the touching boundaries. The distance that the map region moves is set to a constant $d = (\gamma -1)R$. All points in the subtraction area between the new map region and the old one will be deleted from the \textit{ikd-Tree} by a box-wise delete operation detailed in \ref{subsec:increupdates}.

    	\begin{figure}
    	    \centering
    	    \includegraphics[width=1\columnwidth]{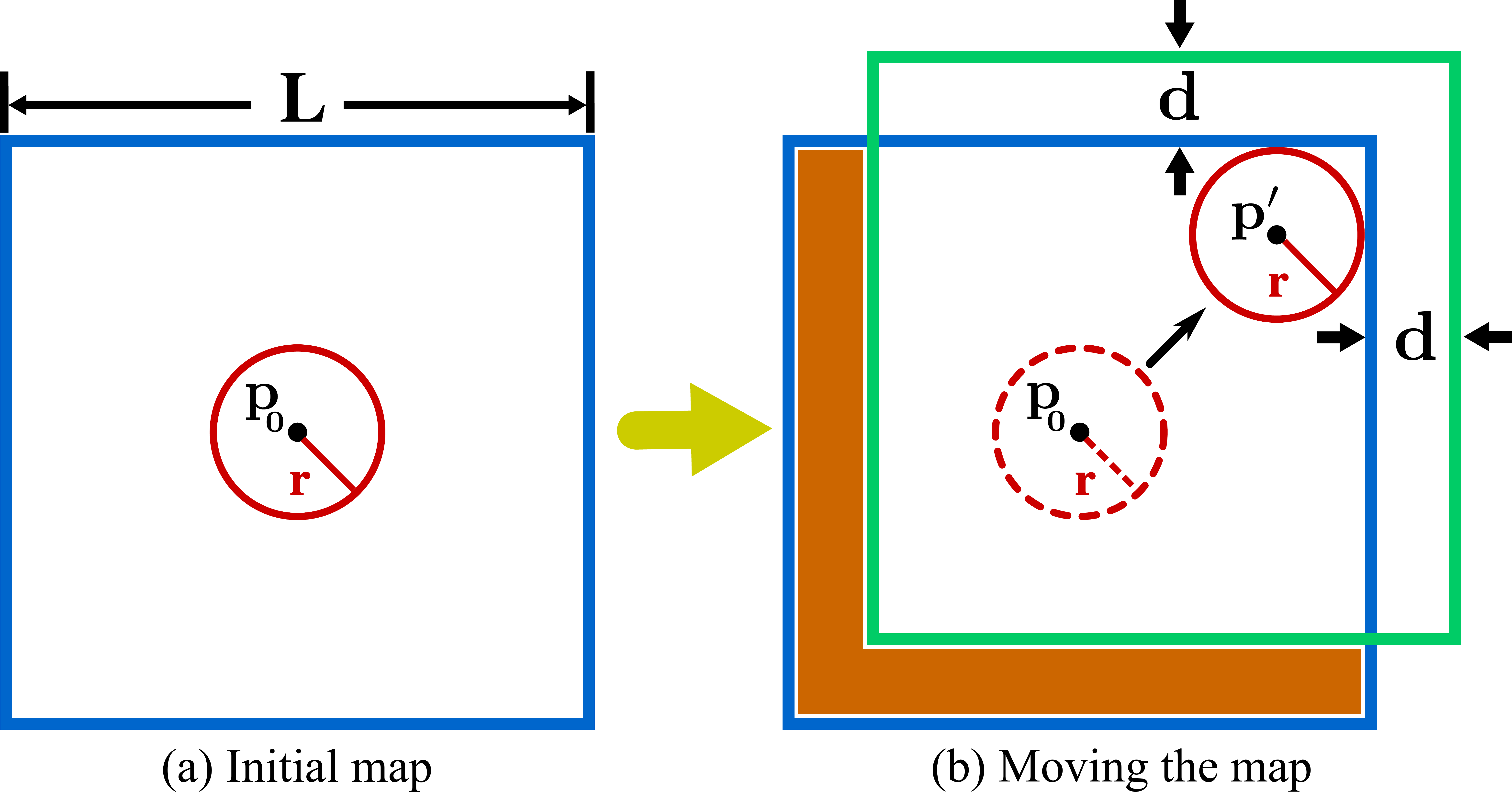}
    	    \caption{2D demonstration of map region management. In (a), the blue rectangle is the initial map region with length $L$. The red circle is the initial detection area centered at the initial LiDAR position $\mathbf{p}_0$. In (b), the detection area (dashed red circle) moves to a new position $\mathbf{p'}$(circle with solid red line) where the map boundaries are touched. The map region is moved to a new position (green rectangle) by distance $d$. The points in the subtraction area (orange area) are removed from the map (i.e., \textit{ikd-Tree}).}
    	    \label{fig:localmap}
    	\end{figure}
        
        \subsection{Tree Structure and Construction}\label{subsec:datastructure}
        \subsubsection{Data Structure}
        The \textit{ikd-Tree} is a binary search tree. The attributes of a tree node in \textit{ikd-Tree} is presented in \textbf{Data Structure}. Different from many existing implementations of k-d trees which  store a ``bucket'' of points only on leaf nodes \cite{arya1998ann,procopiuc2003bkd,libnabo2012,muja2009flannfast,blanco2014nanoflann}, our \textit{ikd-Tree} stores points on both leaf nodes and internal nodes to better support dynamic point insertion and tree re-balancing. Such storing mode has also shown to be more efficient in $k$NN search when a single k-d tree is used \cite{friedman1977search}, which is the case of our \textit{ikd-Tree}. Since a point corresponds to a single node on the \textit{ikd-Tree}, we will use points and nodes interchangeably. The point information (e.g., point coordinates, intensity) are stored in \texttt{point}. The attributes \texttt{leftchild} and \texttt{rightchild} are pointers to its left and right child node, respectively. The division axis to split the space is recorded in \texttt{axis}. The number of tree nodes, including both valid and invalid nodes, of the (sub-)tree rooted at the current node is maintained in attribute \texttt{treesize}. When points are removed from the map, the nodes are not deleted from the tree immediately, but only setting the boolean variable \texttt{deleted} to be true (see Section. \ref{subsubsec:boxdelete} for details). If the entire (sub-)tree rooted at the current node is removed, \texttt{treedeleted} is set to true. The number of points deleted from the (sub-)tree is summed up into attribute \texttt{invalidnum}. The attribute \texttt{range} records the range information of the points on the (sub-)tree, which is interpreted as a circumscribed axis-aligned cuboid containing all the points. The circumscribed cuboid is represented by its two diagonal vertices with minimal and maximal coordinates on each dimension, respectively.
    	\begin{algorithm}
        \SetAlgoLined
        \NoCaptionOfAlgo
        \SetKwProg{kdtreenode}{Struct}{:}{end}
        \kdtreenode{\text{TreeNode}}{
        PointType \texttt{point};\\ 
        TreeNode * \texttt{leftchild}, * \texttt{rightchild};\\
        int \texttt{axis};\\
        int \texttt{treesize}, \texttt{invalidnum};\\
        bool \texttt{deleted}, \texttt{treedeleted};\\        
        CuboidVertices \texttt{range};
        }
        \caption{\textbf{Data Structure}: Tree node structure}
        \end{algorithm}        
    	\subsubsection{Construction}\label{subsubsec:construct}
    	Building the \textit{ikd-Tree} is similar to building a static k-d tree in \cite{bentley1975kdtree}. The \textit{ikd-Tree} splits the space at the median point along the longest dimension recursively until there is only one point in the subspace. The attributes in \textbf{Data Structure} are initialized during the construction, including calculating the tree size and range information of (sub-)trees.  
     
    	\subsection{Incremental Updates}\label{subsec:increupdates}
    	
    	The incremental updates on \textit{ikd-Tree} refer to incremental operations followed by dynamic re-balancing detailed in Section. \ref{sec:balance}. Two types of incremental operations are supported: point-wise operations and box-wise operations. The point-wise operations insert, delete or re-insert a single point to/from the k-d tree, while the box-wise operations insert, delete or re-insert all points in a given axis-aligned cuboid. In both cases, the point insertion is further integrated with on-tree downsampling, which maintains the map at a pre-determined resolution. In this paper, we only explain the point-wise insertion and box-wise delete as they are required by the map management of FAST-LIO2. Readers can refer to our open-source full implementation of {\it ikd-Tree} at Github repository\footnotemark[3] and technical documents contained therein for more details. 
    	    	

    	\subsubsection{Point Insertion with On-tree Downsampling} 
\addtolength{\tabcolsep}{-0.2em}    	    	
\begin{table}[tb]
    \renewcommand{\arraystretch}{1.5}    	
	\caption{Attributes Initialization of a New Tree Node to Insert}
	\label{tab:initialize}
	\centering
\begin{threeparttable}
\begin{tabular}{llll}
\toprule
\multicolumn{1}{c}{Attribute} & \multicolumn{1}{c}{Value}                     & \multicolumn{1}{c}{Attribute}&\multicolumn{1}{c}{Value} \\\cmidrule(l){1-2}\cmidrule(l){3-4}
\texttt{point}                & $\mathbf{p}$              & \texttt{axis}\tnote{1}                 & (\texttt{father}.\texttt{axis} + 1) \textbf{mod} k \\
\texttt{leftchild}            & NULL                      & \texttt{rightchild}           & NULL                                               \\
\texttt{treesize}             & 1                         & \texttt{invalidnum}           & 0                                                  \\
\texttt{deleted}              & false                     & \texttt{treedeleted}          & false                                              \\
\texttt{range}\tnote{2}                & $[\mathbf{p},\mathbf{p}]$ &                               &                                                    \\ \bottomrule
\end{tabular}
\begin{tablenotes}
\footnotesize
\item[1] The \textit{axis} is initialized using the division axis of its father node.
\item[2] The cuboid is initialized by setting minimal and maximal vertices as the point to insert.
\end{tablenotes}
\end{threeparttable}
	    \renewcommand{\arraystretch}{1}    	        
\end{table}    	        
\addtolength{\tabcolsep}{0.2em}      	
        In consideration of robotic applications, our \textit{ikd-Tree} supports simultaneous point insertion and map downsampling. The algorithm is detailed in {\bf Algorithm \ref{alg:downsample}}. For a given point $\mathbf{p}$ in $\{{}^G \bar{\mathbf p}_{j}\}$ from the state estimation module (see {\bf Algorithm \ref{alg:state_estimation}}) and downsample resolution $l$, the algorithm partitions the space evenly into cubes of length $l$, then the cube $\mathbf{C}_D$ that contains the point $\mathbf{p}$ is found (Line 2). The algorithm only keeps the point that is nearest to the center $\mathbf{p}_{center}$ of $\mathbf{C}_D$ (Line 3). This is achieved by firstly searching all points contained in $\mathbf{C}_D$ on the k-d tree and stores them in a point array $V$ together with the new point $\mathbf{p}$ (Line 4-5). The nearest point $\mathbf{p}_{nearest}$ is obtained by comparing the distances of each point in $V$ to the center $\mathbf{p}_{center}$ (Line 6). Then existing points in $\mathbf{C}_D$ are deleted (Line 7), after which the nearest point $\mathbf{p}_{nearest}$ is inserted into the k-d tree (Line 8). The implementation of box-wise search is similar to the box-wise delete as introduced in Section.  \ref{subsubsec:boxdelete}.
        
    	The point insertion (Line 11-24) on the \textit{ikd-Tree} is implemented recursively. The algorithm searches down from the root node until an empty node is found to append a new node (Line 12-14). The attributes of the new leaf node are initialized as Table \ref{tab:initialize}.  At each non-empty node, the new point is compared with the point stored on the tree node along the division axis for further recursion (Line 15-20). The attributes (e.g., \texttt{treesize}, \texttt{range}) of those visited nodes are updated with the latest information (Line 21) as introduced in Section. \ref{subsubsec:attr}. A balance criterion is checked and maintained for sub-trees updated with the new point to keep the balance property of \textit{ikd-Tree} (Line 22) as detailed in Section. \ref{sec:balance}.

        \setcounter{algocf}{1}         
        \begin{algorithm}[t]
        \SetAlgoLined
		\SetKwInOut{Input}{Input}
        \Input{Downsample Resolution $l$,\\ New Point to Insert $\mathbf{p}$,\\ Switch of Parallelly Re-building $SW$}
            \SetKwProg{Fn}{Function}{}{}
            \SetKwFunction{boxsearch}{BoxwiseSearch}
            \SetKwFunction{boxdelete}{BoxwiseDelete}            
            \SetKwFunction{insertpoint}{Insert}
            \SetKwFunction{nearest}{FindNearest}
            \SetKwFunction{getcenter}{Center}
            \SetKwFunction{cube}{FindCube}
            \SetKwFunction{initialize}{Initialize}
            \SetKwFunction{update}{AttributeUpdate}
            \SetKwFunction{maintain}{Rebalance}
            \SetKwProg{Alg}{Algorithm Start}{}{}
			\Alg{}{
	            $\mathbf{C}_D\leftarrow$ \cube{$l,\mathbf{p}$} \\
	            $\mathbf{p}_{center}\leftarrow$ \getcenter{$\mathbf{C}_D$};\\
	            $V\leftarrow$ \boxsearch{\texttt{RootNode},$\mathbf{C}_D$};\\
	            $V.push(\mathbf{p})$;\\
	            $\mathbf{p}_{nearest}\leftarrow$ \nearest$(V,\mathbf{p}_{center})$;\\
	            \boxdelete{\texttt{RootNode},$\mathbf{C}_D$}\\
	            \insertpoint{\texttt{RootNode},$\mathbf{p}_{nearest}$,\texttt{NULL},\texttt{SW}};}
	        \textbf{Algorithm End}\\
	         \textbf{ }\\
		\Fn{\insertpoint{\texttt{T}, $\mathbf{p}$, \texttt{father},\texttt{SW}}}{
			\eIf{\texttt{T} is empty}{
				\initialize{\texttt{T,$\mathbf{p}$,\texttt{father}}};\\
			}{			
				\texttt{ax} $\leftarrow$ \texttt{T}.\texttt{axis};\\
				\eIf{$\mathbf{p}[$\texttt{ax}$]<$ \texttt{T}.\texttt{point}$[$\texttt{ax}$]$}{
					\insertpoint{\texttt{T}.\texttt{leftchild},$\mathbf{p}$,\texttt{T},\texttt{SW}};
				}{
					\insertpoint{\texttt{T}.\texttt{rightchild},$\mathbf{p}$,\texttt{T},\texttt{SW}};
               }
               	\update{\texttt{T}};\\
				\maintain{\texttt{T},\texttt{SW}};               
           }
		}
        \textbf{End Function}		
        \caption{Point Insertion with On-tree Downsampling}
        \label{alg:downsample}
        \end{algorithm}          
        \subsubsection{Box-wise Delete using Lazy Labels} \label{subsubsec:boxdelete}
        
        In the delete operation, we use a lazy delete strategy. That is, the points are not removed from the tree immediately but only labeled as ``deleted" by setting the attribute \texttt{deleted} to true (see \textbf{Data Structure}, Line 6). If all nodes on the sub-tree rooted at node $T$ have been deleted, the attribute \texttt{treedeleted} of $T$ is set to true. Therefore the attributes \texttt{deleted} and \texttt{treedeleted} are called lazy labels. Points labeled as ``deleted" will be removed from the tree during a re-building process (see Section. \ref{sec:balance}).         
		
		Box-wise delete is implemented utilizing the range information in attribute \texttt{range} and the lazy labels on the tree nodes. As mentioned in \ref{subsec:datastructure}, the attribute \textit{range} is represented by a circumscribed cuboid $\mathbf{C}_T$. The pseudo-code is shown in \textbf{Algorithm \ref{alg:boxwise}}. Given the cuboid of points $\mathbf{C}_O$ to be deleted from a (sub-)tree rooted at $T$, the algorithm searches down the tree recursively and compares the circumscribed cuboid $\mathbf{C}_T$ with the given cuboid $\mathbf{C}_O$. If there is no intersection between $\mathbf{C}_T$ and $\mathbf{C}_O$, the recursion returns directly without updating the tree (Line 2). If the circumscribed cuboid $\mathbf{C}_T$ is fully contained in the given cuboid $\mathbf{C}_O$, the box-wise delete set attributes \texttt{deleted} and \texttt{treedeleted} to true (Line 5). As all points on the (sub-)tree are deleted, the attribute \texttt{invalidnum} is equal to the \texttt{treesize} (Line 6). For the condition that $\mathbf{C}_T$ intersects but not contained in $\mathbf{C}_O$, the current point $\mathbf{p}$ is firstly deleted from the tree if it is contained in $\mathbf{C}_O$ (Line 9), after which the algorithm looks into the child nodes recursively (Line 10-11). The attribute update of the current node $T$ and the balance maintenance is applied after the box-wise delete operation (Line 12-13).

        \begin{algorithm}[t]
        \SetAlgoLined

            \SetKwProg{Fn}{Function}{}{}
			\SetKwInOut{Input}{Input}
			\SetKwInOut{Output}{Output}
			\Input{Operation Cuboid $\mathbf{C}_O$,\\ k-d Tree Node $T$,\\ Switch of Parallelly Re-building $SW$}            
            \SetKwFunction{boxwiseop}{BoxwiseDelete}

            \SetKwFunction{update}{AttributeUpdate}
            \SetKwFunction{maintain}{Rebalance}
                    
            \SetKwProg{Fn}{Function}{}{}
            \Fn{\boxwiseop{\texttt{T},$\mathbf{C}_O$,\texttt{SW}}}{
                $\mathbf{C}_T \leftarrow T.range$;\\
                \lIf{$\mathbf{C}_T \cap \mathbf{C}_O = \varnothing$ }{\textbf{return}}

                \eIf{$\mathbf{C}_T \subseteqq \mathbf{C}_O$}{
                    \texttt{T}.\texttt{treedelete}, \texttt{T}.\texttt{delete} $\leftarrow$ true;\\
                    \texttt{T}.\texttt{invalidnum} = \texttt{T}.\texttt{treesize};
                }{
	                $\mathbf{p}\leftarrow$ \texttt{T}.\texttt{point};\\
	                \lIf{$\mathbf{p}\subset \mathbf{C}_O$}{\texttt{T}.\texttt{treedelete} = true}
	                \boxwiseop{\texttt{T.leftchild},$\mathbf{C}_O$,\texttt{SW}};\\
	                \boxwiseop{\texttt{T.rightchild},$\mathbf{C}_O$,\texttt{SW}};\\
	                \update{\texttt{T}};\\
					\maintain{\texttt{T},\texttt{SW}};                
                }

            }
            \textbf{End Function}\\
        \caption{Box-wise Delete}
        \label{alg:boxwise}         
        \end{algorithm}           
        

        \subsubsection{Attribute Update} \label{subsubsec:attr}
        After each incremental operation, attributes of the visited nodes are updated with the latest information using function \texttt{AttributeUpdate}. The function calculates the attributes \texttt{treesize} and \texttt{invalidnum} by summarizing the corresponding attributes on its two child nodes and the point information on itself; the attribute \texttt{range} is determined by merging the range information of the two child nodes and the point information stored on it; \texttt{treedeleted} is set true if the \texttt{treedeleted} of both child nodes are true and the node itself is deleted.  
    	
\subsection{Re-balancing} \label{sec:balance}
    	The \textit{ikd-Tree} actively monitors the balance property after each incremental operation and dynamically re-balances itself by re-building only the relevant sub-trees.
    	
    	\subsubsection{Balancing Criterion}
    	The balancing criterion is composed of two sub-criterions: $\alpha$-balanced criterion and $\alpha$-deleted criterion. Suppose a sub-tree of the \textit{ikd-Tree} is rooted at $T$. The sub-tree is $\alpha$-balanced if and only if it satisfies the following condition:
        \addtolength{\abovedisplayskip}{-0.1cm}
        \addtolength{\belowdisplayskip}{-0.1cm}     	
    	\begin{equation}
    	\begin{aligned}
    	    S(T.leftchild) &< \alpha_{bal} \Big(S(T)-1\Big)\\
    	    S(T.rightchild) &< \alpha_{bal} \Big(S(T)-1\Big)
    	\end{aligned}
    	\label{eq:abalance}
    	\end{equation}
    	where $\alpha_{bal} \in (0.5,1)$ and $S(T)$ is the \texttt{treesize} attribute of the node $T$. 
    	
    	The $\alpha$-deleted criterion of a sub-tree rooted at $T$ is
    	\begin{equation}
    	    \begin{aligned}
        	    I(T)< \alpha_{del} S(T)
    	    \end{aligned}
    	\label{eq:adeleted}
    	\end{equation}
    	where $\alpha_{del} \in(0,1)$ and $I(T)$ denotes the number of invalid nodes on the sub-tree (i.e., the attribute \texttt{invalidnum} of node $T$). 
    	
        If a sub-tree of the \textit{ikd-Tree} meets both criteria, the sub-tree is balanced. The entire tree is balanced if all sub-trees are balanced. Violation of either criterion will trigger a re-building process to re-balance that sub-tree: the $\alpha$-balanced criterion maintains the tree's maximum height. It can be easily proved that the maximum height of an $\alpha$-balanced tree is $\log_{1/\alpha_{bal}} (n)$ where $n$ is the tree size; the $\alpha$-deleted criterion ensures invalid nodes (i.e., labeled as ``deleted") on the sub-trees are removed to reduce tree size. Reducing the height and size of the k-d tree allows highly efficient incremental operations and queries in the future. 
    	
    	\subsubsection{Re-build \& Parallel Re-build}\label{subsubsec:rebuild}
        Assuming re-building is triggered on a subtree $\mathcal{T}$ (see Fig. \ref{fig:rebuild}), the sub-tree is firstly flattened into a point storage array $V$. The tree nodes labeled as ``deleted" are discarded during the flattening. A new perfectly balanced k-d tree is then built with all points in $V$ as Section. \ref{subsec:datastructure}. When re-building a large sub-tree on the \textit{ikd-Tree}, a considerable delay could occur and undermine the real-time performance of FAST-LIO2. To preserve high real-time ability, we design a double-thread re-building method. Instead of simply re-building in the second thread, our proposed method avoids information loss and memory conflicts in both threads by an operation logger, thus retaining full accuracy on $k$-nearest neighbor search at all times.
        \begin{figure}
            \centering
            \includegraphics[width = \columnwidth]{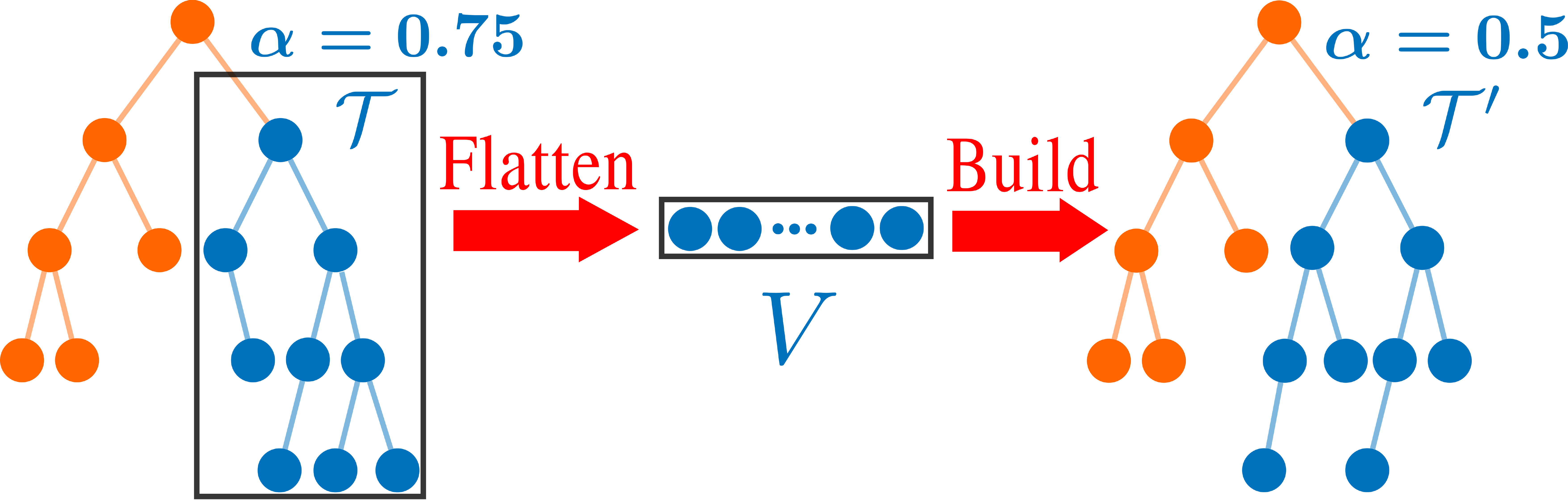}
            \caption{Re-building an unbalanced sub-tree}
            \label{fig:rebuild}
            \vspace{-0.1cm}
        \end{figure}    	
        
        The re-building method is presented in {\bf Algorithm \ref{alg:rebuild}}. When the balance criterion is violated, the sub-tree is re-built in the main thread when its tree size is smaller than a predetermined value $N_{\max}$; Otherwise, the sub-tree is re-built in the second thread. The re-building algorithm on the second thread is shown in function \texttt{ParRebuild}. Denote the sub-tree to re-build in the second thread as $\mathcal{T}$ and its root node as $T$. The second thread will lock all incremental updates (i.e., point insertion and delete) but not queries on this sub-tree (Line 12). Then the second thread copies all valid points contained in the sub-tree $\mathcal{T}$ into a point array $V$ (i.e., flatten) while leaving the original sub-tree unchanged for possible queries during the re-building process (Line 13). After the flattening, the original sub-tree is unlocked for the main thread to take further requests of incremental updates (Line 14). These requests will be simultaneously recorded in a queue named operation logger. Once the second thread completes building a new balanced k-d tree $\mathcal{T}'$ from the point array $V$ (Line 15), the recorded update requests will be performed again on  $\mathcal{T}'$ by function \texttt{IncrementalUpdates} (Line 16-18). Note that the parallel re-building switch is set to false as it is already in the second thread. After all pending requests are processed, the point information on the original sub-tree $\mathcal{T}$ is completely the same as that on the new sub-tree $\mathcal{T}'$ except that the new sub-tree is more balanced than the original one in the tree structure. The algorithm locks the node $T$ from incremental updates and queries and replaces it with the new one $T'$ (Line 20-22). Finally, the algorithm frees the memory of the original sub-tree (Line 23). This design ensures that during the re-building process in the second thread, the mapping process in the main thread proceeds still at the odometry rate without any interruption, albeit at a lower efficiency due to the temporarily unbalanced k-d tree structure. We should note that \texttt{LockUpdates} does not block queries, which can be conducted parallelly in the main thread. In contrast, \texttt{LockAll} blocks all access, including queries, but it finishes very quickly (i.e., only one instruction), allowing timely queries in the main thread. The function \texttt{LockUpdates} and \texttt{LockAll} are implemented by mutual exclusion (mutex).
        \begin{algorithm}[t]
        \SetAlgoLined
		\SetKwInOut{Input}{Input}
		\Input{Root node $T$ of (sub-) tree $\mathcal{T}$ for re-building,\\ Re-build Switch \texttt{SW}}
        \SetKwFunction{build}{Build}
        \SetKwFunction{flatten}{Flatten}
        \SetKwFunction{lockop}{LockUpdates}
        \SetKwFunction{lockall}{LockAll}
        \SetKwFunction{unlock}{Unlock}
        \SetKwFunction{release}{Free}
        \SetKwFunction{runop}{IncrementalUpdates}
        \SetKwFunction{sleep}{Sleep}
        \SetKwFunction{maintain}{Rebalance}
        \SetKwFunction{check}{ViolateCriterion}
		\SetKwFunction{rebuild}{Rebuild}
		\SetKwFunction{parrebuild}{ParRebuild}    
		\SetKwFunction{newthread}{ThreadSpawn}  
		\SetKwProg{Fn}{Function}{}{}
		\Fn{\maintain{$T$,\texttt{SW}}}{
			\If{\check{$T$}}{
	        		\eIf{\texttt{T}.\texttt{treesize} $<N_{\max}$ \textbf{or} Not \texttt{SW}}{ 
                \rebuild{$T$}}{ 
	       		\newthread{\parrebuild,$T$}
		        }
		    }		
		}
        \textbf{End Function}\\
        \textbf{ }\\
        \Fn{\parrebuild{$T$}}{
            \lockop{$T$};\\
            $V \leftarrow \flatten{T}$;\\
            \unlock{$T$};\\
            $T'\leftarrow \build{V}$;\\
            \ForEach{\texttt{op} \textbf{in} \texttt{OperationLogger}}{
                \runop{$T'$,\texttt{op},\texttt{false}}
            }
            $T_{temp} \leftarrow T$;\\                
            \lockall{$T$};\\                        
            $T \leftarrow T'$;\\
            \unlock{$T$};\\
            \release{$T_{temp}$}; \\
        }
        \textbf{End Function}
        \caption{Rebuild (sub-) tree for re-balancing}
        \label{alg:rebuild} 
        \end{algorithm} 

	\subsection{K-Nearest Neighbor Search}
	    Though being similar to existing implementations in those well-known k-d tree libraries\cite{arya1998ann,libnabo2012,muja2009flannfast}, the nearest search algorithm is thoroughly optimized on the \textit{ikd-Tree}. The range information on the tree nodes is well utilized to speed up our nearest neighbor search using a ``bounds-overlap-ball'' test detailed in \cite{friedman1977search}. A priority queue \texttt{q} is maintained to store the $k$-nearest neighbors so far encountered and their distance to the target point. When recursively searching down the tree from its root node, the minimal distance $d_{\min}$ from the target point to the cuboid $\mathbf{C}_T$ of the tree node is calculated firstly. If the minimal distance $d_{\min}$ is no smaller than the maximal distance in \texttt{q}, there is no need to process the node and its offspring nodes. Furthermore, in FAST-LIO2 (and many other LiDAR odometry), only when the neighbor points are within a given threshold around the target point would be viewed as inliers and hence used in the state estimation, this naturally provides a maximal search distance for a ranged search of $k$-nearest neighbors\cite{libnabo2012}. In either case, the ranged search prunes the algorithm by comparing $d_{\min}$ with the maximal distance, thus reducing the amount of backtracking to improve the time performance. It should be noted that our \textit{ikd-Tree} supports multi-thread $k$-nearest neighbor search for parallel computing architectures.
	
        \subsection{Time Complexity Analysis}\label{subsec:timeanalysis}
        The time complexity of \textit{ikd-Tree} breaks into the time for incremental operations (insertion and delete), re-building, and $k$-nearest neighbor search. Note that all analyses are provided under the assumption of low dimensions (e.g., three dimensions in FAST-LIO2).
        
        \subsubsection{Incremental Operations}
        Since the insertion with on-tree downsampling relies on box-wise delete and box-wise search, the box-wise operations are discussed first. Suppose $n$ denotes the tree size of the \textit{ikd-Tree}, the time complexity of box-wise operations on the \textit{ikd-Tree} is:
        \begin{lemma}
        \textit{Suppose points on the ikd-Tree are in 3-d space $S_x\times S_y\times S_z$ and the operation cuboid is $\mathbf{C}_D = L_x\times L_y \times L_z$. The time complexity of box-wise delete and search of \textbf{Algorithm} \ref{alg:boxwise} with cuboid $\mathbf{C}_D$ is }
        \begin{equation}
            O(H(n)) = 
            \begin{cases}
                O(\log n) &  \text{if} \Delta_{\min} \geqslant \alpha(\frac{2}{3})\text{(*)}\\
                O(n^{1-a-b-c}) & \text{if} \Delta_{\max}\leqslant 1-\alpha(\frac{1}{3})\text{(**)}\\
                O(n^{\alpha(\frac{1}{3})-\Delta_{\min}-\Delta_{\med}}) & \text{if (*) and (**) fail and}\\
                 & \Delta_{\med}<\alpha(\frac{1}{3})-\alpha(\frac{2}{3})\\
                O(n^{\alpha(\frac{2}{3})-\Delta_{\min}}) & \text{otherwise.}
            \end{cases}
        \end{equation} 
        \textit{where $a = \log_n \frac{S_x}{L_x}$, $b = \log_n \frac{S_y}{L_y}$ and $c = \log_{n} \frac{S_z}{L_z}$ with $a,b,c\geqslant 0$. $\Delta_{\min}$, $\Delta_{\med}$ and $\Delta_{\max}$ are the minimal, median and maximal value among $a$, $b$ and $c$. $\alpha (u)$ is the flajolet-puech function with $u\in [0,1]$, where particular value is provided: $\alpha(\frac{1}{3}) = 0.7162$ and $\alpha (\frac{2}{3} )=0.3949$.} 
        \label{Lemma:boxwise}        
        \end{lemma} 
        
        \begin{proof}
            The asymptotic time complexity for range search of an axis-aligned hypercube on a k-d tree is provided in \cite{chanzy2001rangesearch}. The box-wise delete can be viewed as a range search except that lazy labels are attached to the tree nodes, which is $O(1)$. Therefore, the conclusion of range search can be applied to the box-wise delete and search on the \textit{ikd-Tree} which leads to $O(H(n))$. 
        \end{proof}
        
        The time complexity of insertion with on-tree downsampling is given as
        \begin{lemma}
        \textit{The time complexity of point insertion with on-tree downsampling in \textbf{Algorithm} \ref{alg:downsample} on \textit{ikd-Tree} is $O(\log n)$.}
        \label{Lemma:pointwise}        
        \end{lemma}
        \begin{proof} The downsampling method on the \textit{ikd-Tree} is composed of box-wise search and delete followed by the point insertion. By applying \textbf{Lemma} \ref{Lemma:boxwise}, the time complexity of downsample is $O(H(n))$. Generally, the downsample cube $\mathbf{C}_D$ is very small comparing with the entire space. Therefore, the normalized range $\Delta x$, $\Delta y$, and $\Delta z$ are small, and the value of $\Delta_{\min}$ satisfies the condition (*) for the time complexity of $O(\log n)$.
        
        The maximum height of the \textit{ikd-Tree} can be easily proved to be $\log_{1/\alpha_{bal}}(n)$ from Eq. (\ref{eq:abalance}) while that of a static k-d tree is $\log_2 n$. Hence the lemma is directly obtained from \cite{bentley1975kdtree} where the time complexity of point insertion on a k-d tree was proved to be $O(\log n)$. Summarizing the time complexity of both downsample and insertion concludes that the time complexity of insertion with on-tree downsampling is $O(\log n)$.
        \end{proof}
        \subsubsection{Re-build} The time complexity for re-building falls into two types: single-thread re-building and parallel double-thread re-building. In the former case, the re-building is performed by the main thread recursively. Each level takes the time of sorting (i.e., $O(n)$) and the total time over $\log n$ levels is $O(n\log n)$ \cite{bentley1975kdtree} when the dimension $k$ is low. For parallel re-building, the time consumed in the main thread is only flattening (which suspends the main thread from further incremental updates, \textbf{Algorithm \ref{alg:rebuild}}, Line 12-14) and tree update (which takes constant time $O(1)$, \textbf{Algorithm \ref{alg:rebuild}}, Line 20-22) but not building (which is performed in parallel by the second thread, \textbf{Algorithm \ref{alg:rebuild}}, Line 15-18), leading to time complexity of $O(n)$ (viewed from the main thread). In summary, the time complexity of re-building the \textit{ikd-Tree} is $O(n)$ for double-thread parallel re-building and $O(n\log n)$ for single-thread re-building. 
        \subsubsection{K-Nearest Neighbor Search}
        As the maximum height of the \textit{ikd-Tree} is maintained no larger than $\log_{1/\alpha_{bal}}(n)$, where $n$ is the tree size, the time complexity to search down from root to leaf nodes is $O(\log n)$. During the process of searching $k$-nearest neighbors on the tree, the number of backtracking is proportional to a constant $\bar{l}$ which is independent of the tree size\cite{friedman1977search}. Therefore, the expected time complexity to obtain $k$-nearest neighbors on the \textit{ikd-Tree} is $O(\log n)$.

\section{Benchmark Results}\label{sec:benchmark}
In this section, extensive experiments in terms of accuracy, robustness, and computational efficiency are conducted on various open datasets. We first evaluate our data structure, i.e., \textit{ikd-Tree}, against other data structures for $k$NN search on 18 dataset sequences of different sizes. Then in Section. \ref{sec:benchmark_accuracy}, we compare the accuracy and processing time of FAST-LIO2 on 19 sequences. All the sequences are chosen from 5 different datasets collected by both solid-state LiDAR~\cite{liu2020low} and spinning LiDARs. The first dataset is from the work LILI-OM~\cite{li2021towards} and is collected by a solid-state 3D LiDAR Livox Horizon\footnote[4]{\url{https://www.livoxtech.com/horizon}}, which has non-repetitive scan pattern and 81.7 $^\circ$ (Horizontal) × 25.1$^\circ$ (Vertical) FoV, at a typical scan rate of 10 $Hz$, referred to as \textit{lili}. The gyroscope and accelerometer measurements are sampled at 200 $Hz$ by a 6-axis Xsens MTi-670 IMU. The data is recorded in the university campus and urban streets with structured scenes. The second dataset is from the work LIO-SAM~\cite{T2020liosam} in MIT campus and contains several sequences collected by a VLP-16 LiDAR\footnote[5]{\url{https://velodynelidar.com/products/puck-lite/}} sampled at 10 $Hz$ and a MicroStrain 3DM-GX5-25 9-axis IMU sampled at 1000 $Hz$, referred to as \textit{liosam}. It contains different kinds of scenes, including structured buildings and forests on campus.  The third dataset ``\textit{utbm}"~\cite{eu_longterm_dataset} is collected with a human-driving robocar in maximum 50 $km/h$ speed which has two 10 $Hz$ Velodyne HDL-32E LiDAR\footnote[6]{\url{https://velodynelidar.com/products/hdl-32e/}} and 100 $Hz$ Xsens MTi-28A53G25 IMU. In this paper, we only consider the left LiDAR. The fourth dataset ``\textit{ulhk}"~\cite{wen2020urbanloco} contains the 10 $Hz$ LiDAR data from Velodyne HDL-32E and 100 $Hz$ IMU data from a 9-axis Xsens MTi-10 IMU. All the sequences of \textit{utbm} and \textit{ulhk} are collected in structured urban areas by a human-driving vehicle while \textit{ulhk} also contains many moving vehicles. The last one, ``\textit{nclt}"\cite{carlevaris2016university} is a large-scale, long-term autonomy UGV (unmanned ground vehicle) dataset collected in the University of Michigan’s North Campus. The \textit{nclt} dataset contains 10 $Hz$ data from a Velodyne HDL-32E LiDAR and 50 $Hz$ data from Microstrain MS25 IMU. The \textit{nclt} dataset has a much longer duration and amount of data than other datasets and contains several open scenes such as a large open parking lot. The datasets information including the sensors' type and data rate is summarized in Table.\ref{tab:benchmark_datasets}. The details about all the 37 sequences used in this section, including name, duration, and distance, are listed in Table. \ref{tab:append:detail_sequences} of Appendix. \ref{appendix:sequence_list}.
\addtolength{\tabcolsep}{0.5em} 
\begin{table}[t]
\caption{The Datasets for Benchmark}
\label{tab:benchmark_datasets}
\centering
\begin{threeparttable}
\begin{tabular}{llrllr} 
\toprule
\multirow{2}{*}{} & \multicolumn{2}{c}{LiDAR}              &  & \multicolumn{2}{c}{IMU}             \\ 
\cline{2-3}\cline{5-6}
                  & Type        & \multicolumn{1}{l}{Line} &  & Type   & \multicolumn{1}{l}{Rate}  \\ 
\midrule
\textit{lili}     & Solid-state & —                        &  & 6-axis & 200 $Hz$                        \\
\textit{utbm}     & Spinning    & 32                       &  & 6-axis & 100 $Hz$                        \\
\textit{ulhk}     & Spinning    & 32                       &  & 9-axis & 100 $Hz$                       \\
\textit{nclt}     & Spinning    & 32                       &  & 9-axis & 100 $Hz$                       \\
\textit{liosam}   & Spinning    & 16                       &  & 9-axis & 1000 $Hz$                       \\
\bottomrule
\end{tabular}
\begin{tablenotes}
\footnotesize
\item[1] In order to make LIO-SAM works, the IMU rate in dataset \textit{nclt} is increased from 50 to 100 $Hz$ through zero-order interpolation.
\end{tablenotes}
\end{threeparttable}
\end{table}
\addtolength{\tabcolsep}{-0.5em} 

\subsection{Implementation}\label{subsec:nominal}
We implemented the proposed FAST-LIO2 system in C++ and Robots Operating System (ROS). The iterated Kalman filter is implemented based on the {\it IKFOM} toolbox presented in our previous work\cite{he2021embedding}. In the default configuration, the local map size $L$ is chosen as 1000 $m$, and the LiDAR raw points are directly fed into state estimation after a 1:4 (one out of four LiDAR points) temporal downsampling. Besides, the spatial downsample resolution (see {\bf Algorithm \ref{alg:downsample}}) is set to $l = 0.5 m$ for all the experiments. The parameter of \textit{ikd-Tree} is set to $\alpha_{bal} = 0.6$, $\alpha_{del}=0.5$ and $N_{\max}=1500$. The computation platform for benchmark comparison is a lightweight UAV onboard computer: DJI Manifold 2-C\footnote[7]{\url{https://www.dji.com/cn/manifold-2/specs}} with a 1.8 $GHz$ quad-core Intel i7-8550U CPU and 8 $GB$ RAM. For FAST-LIO2, we also test it on an ARM processor that is typically used in embedded systems with reduced power and cost. The ARM platform is Khadas VIM3\footnote[8]{\url{https://www.khadas.com/vim3}} which has a low-power 2.2 $GHz$ quad-core Cortex-A73 CPU and 4 $GB$ RAM, denoted as the keyword ``ARM". We denote ``FAST-LIO2 (ARM)" as the implementation of FAST-LIO2 on the ARM-based platform.

\subsection{Data structure Evaluation}\label{sec:benchmark_ds}
\subsubsection{Evaluation Setup}
We select three state-of-art implementations of dynamic data structure to compare with our \textit{ikd-Tree}: The boost geometry library implementation of R$^\ast$-tree\cite{boost_geometry}, the Point Cloud Library implementation of octree\cite{rusu2011pcl} and the \textit{nanoflann}\cite{blanco2014nanoflann} implementation of dynamic k-d tree. These tree data structure implementations are chosen because of their high implementation efficiency. Moreover, they support dynamic operations (i.e., point insertion, delete) and range (or radius) search that is necessary to be integrated with FAST-LIO2 for a fair comparison with {\it ikd-Tree}. For the map downsampling, since the other data structures do not support on-tree downsampling as \textit{ikd-Tree}, we apply a similar approach as detailed in \ref{subsec:increupdates} by utilizing their ability of range search (for octree and R$^\ast$-tree) or radius search (for \textit{nanoflann} k-d tree). More specifically, for octree and R$^\ast$-tree, their range search directly returns points within the target cuboid $\mathbf C_D$ (see \textbf{Algorithm \ref{alg:downsample}}). For \textit{nanoflann} k-d tree, the target cuboid $\mathbf C_D$ is firstly split into small cubes whose edge length equals the minimal edge length of the cuboid. Then the points inside the circumcircle of each small cube are obtained by radius search, after which points outside the cuboid are filtered out via a linear approach while points inside the target cuboid $\mathbf C_D$ remain. Finally, similar to \textbf{Algorithm \ref{alg:downsample}}, points in $\mathbf C_D$ other than the nearest point to the center are removed from the map. For the box-wise delete operation required by map move (see Section. \ref{sec:map_manage}), it is achieved by removing points within the specified cuboid one by one according to the point indices obtained from the respective range or radius search.

All the four data structure implementations are integrated with FAST-LIO2 and their time performance are evaluated on 18 sequences of different sizes. We run the FAST-LIO2 with each data structure for each sequence and record the time for $k$NN search, point insertion (with map downsampling), box-wise delete due to map move, the number of new scan points, and the number of map points (i.e., tree size) at each step. The number of nearest neighbors to find is $5$. 

\begin{figure}[t]
    \centering
    \includegraphics[width=1.0\columnwidth]{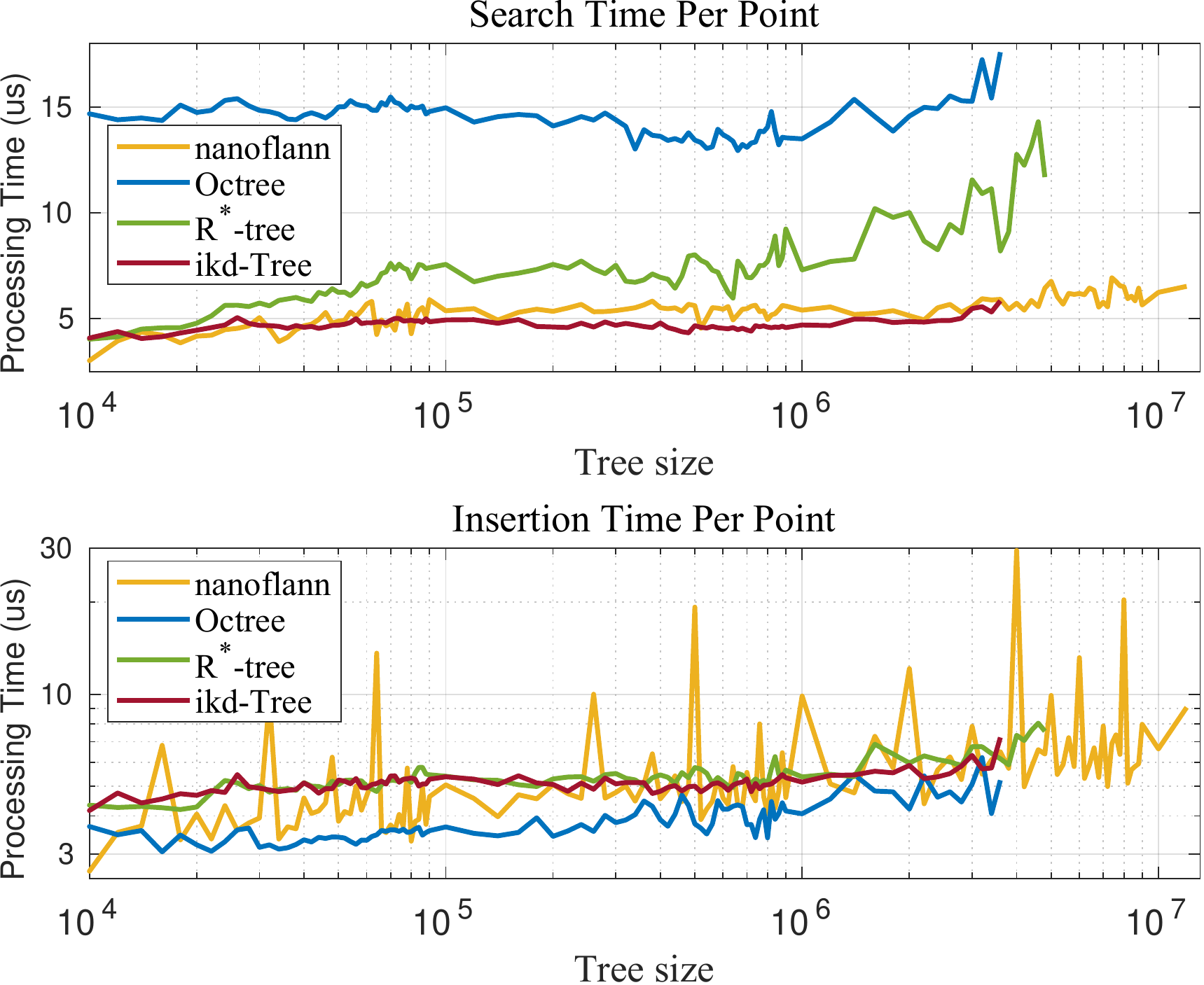}
    \caption{Data structure comparison over different tree size.}
    \label{fig:dscmp_size}
\end{figure}

\subsubsection{Comparison Results}

\begin{table*}[t]
\caption{The Comparison of Average Time Consumption Per Scan on Incremental Updates, $k$NN Search and Total Time}
\label{tab:dscomp}
\centering
\begin{threeparttable}
\begin{tabular}{@{}lrrrrrrrrrrrr@{}}
\toprule
 & \multicolumn{4}{c}{Incremental Update\tnote{1} [ms]}      & \multicolumn{4}{c}{$k$NN Search\tnote{2} [ms]}        & \multicolumn{4}{c}{Total [ms]}                      \\ \cmidrule(l){2-5}\cmidrule(l){6-9}\cmidrule(l){10-13} 
 & ikd-Tree      & nanoflann & Octree        & R$^\ast$-tree & ikd-Tree       & nanoflann & Octree & R$^\ast$-tree & ikd-Tree       & nanoflann & Octree & R$^\ast$-tree \\ \midrule
\textit{utbm\_1} & 3.23          & 3.43      & \textbf{2.12} & 3.94          & \textbf{15.19} & 15.80     & 42.88  & 22.56         & \textbf{18.42} & 19.22     & 45.00  & 26.50         \\
\textit{utbm\_2} & 3.40          & 3.65      & \textbf{2.24} & 4.18          & \textbf{15.52} & 16.09     & 44.70  & 23.46         & \textbf{18.93} & 19.75     & 46.94  & 27.64         \\
\textit{utbm\_3} & 3.77          & 4.17      & \textbf{2.36} & 4.52          & \textbf{16.83} & 18.54     & 45.72  & 23.12         & \textbf{20.60} & 22.70     & 48.08  & 27.64         \\
\textit{utbm\_4} & 3.52          & 3.70      & \textbf{2.26} & 4.32          & \textbf{16.53} & 17.60     & 44.80  & 24.74         & \textbf{20.06} & 21.30     & 47.06  & 29.06         \\
\textit{utbm\_5} & 3.34          & 3.60      & \textbf{2.21} & 4.21          & \textbf{15.51} & 16.65     & 45.42  & 23.38         & \textbf{18.85} & 20.25     & 47.63  & 27.58         \\
\textit{utbm\_6} & 3.61          & 4.12      & \textbf{2.34} & 4.60          & \textbf{16.25} & 17.14     & 43.06  & 23.49         & \textbf{19.86} & 21.27     & 45.40  & 28.09         \\
\textit{utbm\_7} & 3.82          & 4.62      & \textbf{2.55} & 5.26          & \textbf{15.42} & 16.97     & 42.06  & 25.87         & \textbf{19.24} & 21.59     & 44.61  & 31.13         \\
\textit{ulhk\_1} & 1.97          & 1.87      & \textbf{1.12} & 2.30          & \textbf{18.23} & 21.73     & 48.30  & 23.45         & \textbf{20.20} & 23.60     & 49.43  & 25.75         \\
\textit{ulhk\_2} & 3.51          & 3.43      & \textbf{2.32} & 4.23          & \textbf{22.26} & 26.07     & 64.56  & 31.75         & \textbf{25.77} & 29.49     & 66.88  & 35.98         \\
\textit{ulhk\_3} & 1.60          & 1.58      & \textbf{1.10} & 1.93          & \textbf{13.62} & 14.87     & 42.65  & 20.49         & \textbf{15.22} & 16.45     & 43.74  & 22.42         \\
\textit{nclt\_1} & 1.14          & 1.59      & \textbf{0.99} & 2.07          & \textbf{14.50} & 18.83     & 41.58  & 28.07         & \textbf{15.64} & 20.41     & 42.57  & 30.14         \\
\textit{nclt\_2} & \textbf{1.35} & 2.04      & 1.36          & 2.66          & \textbf{14.68} & 18.99     & 46.56  & 29.20         & \textbf{16.03} & 21.03     & 47.91  & 31.86         \\
\textit{nclt\_3} & \textbf{1.00} & 1.42      & 1.03          & 2.20          & \textbf{14.41} & 19.25     & 46.19  & 30.10         & \textbf{15.42} & 20.67     & 47.22  & 32.29         \\
\textit{lili\_1} & 1.41          & 1.42      & \textbf{0.83} & 1.79          & \textbf{9.20}  & 9.71      & 26.31  & 12.65         & \textbf{10.61} & 11.13     & 27.15  & 14.44         \\
\textit{lili\_2} & 1.53          & 1.50      & \textbf{0.84} & 1.81          & \textbf{8.94}  & 9.27      & 26.18  & 13.43         & \textbf{10.47} & 10.77     & 27.02  & 15.24         \\
\textit{lili\_3} & 1.10          & 1.14      & \textbf{0.63} & 1.38          & \textbf{8.46}  & 8.87      & 25.45  & 13.18         & \textbf{9.57}  & 10.00     & 26.08  & 14.56         \\
\textit{lili\_4} & 0.96          & 0.99      & \textbf{0.62} & 1.39          & \textbf{10.69} & 11.97     & 32.55  & 15.71         & \textbf{11.65} & 12.96     & 33.17  & 17.10         \\
\textit{lili\_5} & 1.22          & 1.28      & \textbf{0.80} & 1.63          & \textbf{10.23} & 11.34     & 33.53  & 12.78         & \textbf{11.45} & 12.62     & 34.33  & 14.41         \\ \bottomrule
\end{tabular}
\begin{tablenotes}
\footnotesize
\item[1] Average time consumption per scan of incremental updates, including point-wise insertion with on-tree downsampling and box-wise delete.
\item[2] Average time consumption per scan of single-thread $k$NN search.
\end{tablenotes}
\end{threeparttable}
\end{table*}

We first compare the time consumption of point insertion (with map downsampling) and $k$NN search at different tree sizes across all the 18 sequences. For each evaluated tree size $S$, we collect the processing time at tree size of $[S-5\%S, S+5\%S]$ to obtain a sufficient number of samples.  Fig. \ref{fig:dscmp_size} shows the average time consumption of insertion and $k$NN search per single target point. The octree achieves the best performance in point insertion, albeit the gap with the other is small (below 1 $\mu s$), but its inquiry time is much higher due to the unbalanced tree structure. For \textit{nanoflann} k-d tree, the insertion time is often slightly shorter than the \textit{ikd-Tree} and R$^\ast$-tree, but huge peaks occasionally occur due to its logarithmic structure of organizing a series of k-d trees. Such peaks severely degrade the real-time ability, especially when maintaining a large map. For $k$-nearest neighbor search, \textit{nanoflann} k-d tree consumes slightly higher time than our \textit{ikd-Tree}, especially when the tree size becomes large ($10^5 \sim 10^6$). The R$^\ast$-tree achieves a similar insertion time with {\it ikd-Tree} but with a significantly higher search time for large tree sizes. Finally, we can see that the time of insertion with on-tree downsampling and $k$NN search of \textit{ikd-Tree} is indeed proportional to $\log n$, which is consistent with the time complexity analysis in Section. \ref{subsec:timeanalysis}. 

For any map data structure to be used in LiDAR odometry and mapping, the total time for map inquiry (i.e., $k$NN search) and incremental map update (i.e., point insertion with downsampling and box-delete due to map move) ultimately affects the real-time ability. This total time is summarized in Table \ref{tab:dscomp}. It is seen that octree performs the best in incremental updates in most datasets, followed closely by the \textit{ikd-Tree} and the \textit{nanoflann} k-d tree. In $k$NN search, the \textit{ikd-Tree} has the best performance while the \textit{ikd-Tree} and \textit{nanoflann} k-d tree outperforms the other two by large margins, which is consistent with the past comparative study\cite{libnabo2012,vermeulen2017comparative}. The \textit{ikd-Tree} achieves the best overall performance among all other data structures. 

We should remark that while the \textit{nanoflann} k-d tree achieves seemly similar performance with {\it ikd-Tree}, the peak insertion time has more profound causes, and its impact on LiDAR odometry and mapping is severe. The \textit{nanoflann} k-d tree deletes a point by only masking it without actually deleting it from the tree. Consequently, even with map downsampling and map move, the deleted points remain on the tree affecting the subsequent inquiry and insertion performance. The resultant tree size grows much quicker than {\it ikd-Tree} and others, a phenomenon also observed from Fig. \ref{fig:dscmp_size}. The effect could be small for short sequences (\textit{ulhk} and \textit{lili}) but becomes evident for long sequences (\textit{utbm} and \textit{nclt}). The tree size of \textit{nanoflann} k-d tree exceeds $6\times10^6$ in \textit{utbm} datasets and $10^7$ in \textit{nclt} datasets, whereas the maximal tree size of \textit{ikd-Tree} reaches $2\times10^6$ and $3.6\times10^6$, respectively. The maximal processing time of incremental updates on \textit{nanoflann} all exceeds 3 $s$ in seven \textit{utbm} datasets and 7 $s$ in three \textit{nclt} datasets while our \textit{ikd-Tree} keeps the maximal processing time at 214.4 $ms$ in \textit{nclt\_2} and smaller than 150 $ms$ in the rest 17 sequences. While this peaked processing time of \textit{nanoflann} does not heavily affect the overall real-time ability due to its low occurrence, it causes a catastrophic delay for subsequent control.

\subsection{Accuracy Evaluation}\label{sec:benchmark_accuracy}
In this section, we compare the overall system FAST-LIO2 against other state-of-the-art LiDAR-inertial odometry and mapping systems, including LILI-OM~\cite{li2021towards}, LIO-SAM \cite{T2020liosam}, and LINS \cite{qin2019lins}. Since FAST-LIO2 is an odometry without any loop detection or correction, for the sake of fair comparison, the loop closure module of LILI-OM and LIO-SAM was deactivated, while all other functions such as sliding window optimization are enabled. We also perform ablation study on FAST-LIO2: to understand the influence of the map size, we run the algorithm in various map sizes $L$ of 2000 $m$, 800 $m$, 600 $m$, besides the default 1000 $m$; to evaluate the effectiveness of direct method against feature-based methods, we add a feature extraction module from FAST-LIO\cite{xu2020fastlio} (optimized for solid-state LiDAR) and BALM\cite{liu2021balm} (optimized for spinning LiDAR). The results are reported under the keyword ``Feature". All the experiments are conducted in the Manifold 2-C platform (Intel).  

We perform evaluations on all the five datasets: {\it lili, lisam, utbm, ulhk}, and {\it nclt}.  Since not all sequences have ground truth (affected by the weather, GPS quality, etc.), we select a total of 19 sequences from the five datasets. These 19 sequences either have a good ground truth trajectory (as recommended by the dataset author) or end at the starting position. Therefore, two criteria, absolute translational error (RMSE) and end-to-end error, are computed and evaluated. 

\subsubsection{RMSE Benchmark}
The RMSE are computed and reported in Table.~\ref{tab:rmse_benchmark}. It is seen that increasing the map size of FAST-LIO2 increases the overall accuracy as the new can is registered to older historical points when the LiDAR revisits a previous place. When the map size is over 2000 $m$, the accuracy increment is not persistent as the odometry drift may cause possible false point match with too old map points, a typical phenomenon of any odometry. Moreover, the direct method outperforms the feature-based variant of FAST-LIO2 in most sequences except for two, {\it nclt\_4} and {\it nclt\_6},  where the difference is tiny and negligible. This proves the effectiveness of the direct method.
 \begin{table*}[t]
\footnotesize
\centering
\caption{Absolute translational errors (RMSE, meters) in sequences with good quality ground truth}
\label{tab:rmse_benchmark}
\begin{threeparttable}
\begin{tabular}{@{}lrrrrrrrrrrrr@{}}
\toprule
 & \textit{utbm\_8} & \textit{utbm\_9} & \textit{utbm\_10} & \textit{ulhk\_4} & \textit{nclt\_4} & \textit{nclt\_5} & \textit{nclt\_6} & \textit{nclt\_7} & \textit{nclt\_8} & \textit{nclt\_9} & \textit{nclt\_10} & \textit{liosam\_1} \\ \midrule
FAST-LIO2 (2000m)   & \textbf{25.3}    & \textbf{51.6}    & 16.89             & 2.57             & 8.63             & 6.66             & 21.01            & 6.59             & 30.59            & 5.72             & 17.14             & 4.62               \\
FAST-LIO2 (1000m)   & 27.29            & \textbf{51.6}    & \textbf{16.8}     & 2.57             & 8.71             & 6.68             & 20.96            & \textbf{6.58}             & \textbf{30.08}   & \textbf{5.56}    & \textbf{16.29}    & \textbf{4.58}      \\
FAST-LIO2 (800m)    & 25.8             & 51.86            & 17.23             & 2.57             & 8.72             & \textbf{6.65}             & 21.03            & 6.99             & 30.74            & 5.95             & 16.73             & \textbf{4.58}      \\
FAST-LIO2 (600m)    & 27.75            & 52.09            & 17.3              & 2.57             & 8.58             & 6.69             & 20.96            & 6.82             & 30.24            & 5.8              & 16.81             & \textbf{4.58}      \\
FAST-LIO2 (Feature) & 27.21            & 53.81            & 22.59             & 2.61             & \textbf{8.5}     & 7.82             & \textbf{20.57}   & 6.77             & 31.17            & 6.09             & 16.61             & 7.85               \\
LILI-OM             & 59.48            & 782.11           & 17.59             & \textbf{2.29}    & 317.77           & 12.42            & 260.76           & 12.17            & 276.74           & 7.39             & 328.87            & 18.78              \\
LIO-SAM             & —\tnote{1}           & —           & —       & 3.52             & 9461             & 7.15             & ×\tnote{2}          & 22.26            & 44.83            & 7.43             & 1077.5            & 4.75               \\
LINS                & 48.17            & 54.35            & 60.48             & 3.11             & 65.95            & 1051             & 243.87           & 378.99           & 106.03           & 11.13            & 2995.9            & 880.92             \\ \bottomrule
\end{tabular}
\begin{tablenotes}
\footnotesize
\item[1] Dataset \textit{utbm} does not produce the attitude quaternion data which is necessary for LIO-SAM, therefore LIO-SAM does not work on all the sequences in \textit{utbm} dataset, denoted as —.
\item[2] × denotes that the system totally failed.
\end{tablenotes}
\end{threeparttable}
\end{table*}
 
Compared with other LIO methods, FAST-LIO2 or its variant achieves the best performances in 18 of all 19 data sequences and is the most robust LIO method among all the experiments. The only exception is on {\it ulhk\_4} where LILI-OM shows slightly higher accuracy than FAST-LIO. Notably, LILI-OM shows very large drift in \textit{utbm\_9}, \textit{nclt\_4}, \textit{nclt\_6}, \textit{nclt\_8} and \textit{nclt\_10}. The reason is that its sliding-window back-end fusion ({\it mapping}) fails as the map point number grows large. Hence its pose estimation relies solely on the front-end {\it odometry} which quickly accumulates the drift. LINS works similarly badly in \textit{nclt\_5}, \textit{nclt\_6}, \textit{nclt\_7}, \textit{nclt\_10}. LIO-SAM also shows large drift at \textit{nclt\_4}, \textit{nclt\_10} due to the failure of back-end factor graph optimization with the very long time and long-distance data. The video of an example, {\it nclt\_10} sequence, is available at \url{https://youtu.be/2OvjGnxszf8}. Besides, on other sequences where LILI-OM, LIO-SAM, and LINS can work normally, their performance is still outperformed by FAST-LIO2 with large margins. Finally, it should be noted that the sequence {\it liosam\_1} is directly drawn from the work LIO-SAM \cite{T2020liosam} so the algorithm has been well-tuned for the data. However, FAST-LIO2 still achieves higher accuracy.

\subsubsection{Drift Benchmark}
The end-to-end errors are reported in Table.~\ref{tab:ee_benchmark}. The overall trend is similar to the RMSE benchmark results. FAST-LIO2 or its variants achieves the lowest drift in 5 of the total 7 sequences. We show an example, {\it ulhk\_6} sequence, in the video available at \url{https://youtu.be/2OvjGnxszf8}. It should be noted that the LILI-OM has tuned parameters for each of their own sequences {\it lili} while parameters of FAST-LIO2 are kept the same among all the sequences. LIO-SAM shows good performance in its own sequences {\it liosam\_2} and {\it liosam\_3} but cannot keep it on other sequences such as {\it ulhk}. The LINS performs worse than LIO-SAM in {\it liosam} and {\it ulhk} datasets and failed in \textit{liosam\_2} (garden sequence from \cite{T2020liosam}) because the two sequences are recorded with large rotation speeds while the feature points used by LINS are too few. Also, in most of the sequences, the feature-based FAST-LIO performs similarly to the direct method except for the sequence {\it lili\_7}, which contains many trees and large open areas that feature extraction will remove many effective points from trees and faraway buildings.

\begin{table}[t]
\footnotesize
\centering
\caption{End to end errors (meters)}
\label{tab:ee_benchmark}
\begin{threeparttable}
\begin{tabular}{@{}lrrrrrrr@{}}
\toprule
 &  \rotatebox{90}{\textit{lili\_6}} & \rotatebox{90}{\textit{lili\_7}} & \rotatebox{90}{\textit{lili\_8}} & \rotatebox{90}{\textit{ulhk\_5}} & \rotatebox{90}{\textit{ulhk\_6}} & \rotatebox{90}{\textit{liosam\_2}} & \rotatebox{90}{\textit{liosam\_3}} \\ \midrule
\begin{tabular}[c]{@{}l@{}}FAST-LIO2\\ (2000m)\end{tabular}   &  0.14                    & \multicolumn{1}{l}{1.92} & 21.35                              & 0.33             & 0.12                   & \textbf{\textless 0.1} & 9.23               \\
\begin{tabular}[c]{@{}l@{}}FAST-LIO2\\ (1000m)\end{tabular}   & \textbf{\textless 0.1}   & 1.63                     & \multicolumn{1}{l}{{17.39}} & 0.39             & \textbf{\textless 0.1} & \textbf{\textless 0.1} & 9.50               \\
\begin{tabular}[c]{@{}l@{}}FAST-LIO2\\ (800m)\end{tabular}    & \textbf{\textless 0.1}   & 1.88                     & \multicolumn{1}{l}{21.59}          & 0.40             & \textbf{\textless 0.1} & \textbf{\textless 0.1} & 9.49               \\
\begin{tabular}[c]{@{}l@{}}FAST-LIO2\\ (600m)\end{tabular}    & 0.22                     & \textbf{1.37}            & \multicolumn{1}{l}{23.74}          & 0.39             & \textbf{\textless 0.1} & \textbf{\textless 0.1} & 9.23               \\
\begin{tabular}[c]{@{}l@{}}FAST-LIO2\\ (Feature)\end{tabular} & 0.20                     & \multicolumn{1}{l}{3.89} & \multicolumn{1}{l}{21.99}          & \textbf{0.32}    & \textbf{\textless 0.1} & \textbf{\textless 0.1} & {12.11}     \\
LILI-OM                                                       & 0.80                     & 4.13                     & \textbf{15.60}                     & {1.84}    & 7.89                   & 1.95                   & 13.79              \\
LIO-SAM                                                       & —\tnote{1}              & —                        & —                                  & 0.83             & 2.88                   & \textbf{\textless 0.1} & \textbf{8.61}      \\
LINS                                                          & —                        & —                        & —                                  & 0.90             & 6.92                   & ×\tnote{2}              & 29.90              \\ \bottomrule
\end{tabular}
\begin{tablenotes}
\footnotesize
\item[1] Since the LIO-SAM and LINS are both developed only for spinning LiDAR, they do not work on the \textit{lili} dataset which is recorded by a solid-state LiDAR Livox Horizon.
\item[2] × denotes that the system totally failed.
\end{tablenotes}
\end{threeparttable}
\end{table}

\begin{table*}[h]
\scriptsize
\centering
\caption{The Average Processing Time per Scan Benchmark in Milliseconds}
\label{tab:runtime_benchmark}
\begin{tabular}{@{}lrrrrrrrrrrrr@{}}
\toprule
& \rotatebox{0}{\begin{tabular}[c]{@{}r@{}}FAST-LIO2\\      (2000)\end{tabular}} & \rotatebox{0}{\begin{tabular}[c]{@{}r@{}}FAST-LIO2\\      (1000)\end{tabular}} & \rotatebox{0}{\begin{tabular}[c]{@{}r@{}}FAST-LIO2\\      (800)\end{tabular}} & \rotatebox{0}{\begin{tabular}[c]{@{}r@{}}FAST-LIO2\\      (600)\end{tabular}} & \rotatebox{0}{\begin{tabular}[c]{@{}r@{}}FAST-LIO2\\      (Feature)\end{tabular}} & \rotatebox{0}{\begin{tabular}[c]{@{}r@{}}FAST-LIO2\\      (ARM)\end{tabular}} & \multicolumn{2}{c}{LILI-OM} & \multicolumn{2}{c}{LIO-SAM}                         & \multicolumn{2}{c}{LINS}                            \\ \midrule
 & Total                                                           & Total                                                           & Total                                                          & Total                                                          & Total                                                          & Total                                                              & Odo.         & Map.         & Odo.                     & Map.                     & Odo.                     & Map.                     \\
\textit{lili\_6}                                & 13.15                                                           & \textbf{12.56}                                                  & 13.22                                                          & 15.92                                                          & 15.35                                                              & 45.58                                                          & 68.95        & 58.46        & {\color[HTML]{333333} —} & {\color[HTML]{333333} —} & {\color[HTML]{333333} —} & {\color[HTML]{333333} —} \\
\textit{lili\_7}                                & \textbf{16.93}                                                  & 17.61                                                           & 20.39                                                          & 19.72                                                          & 21.13                                                              & 65.89                                                          & 40.01        & 83.71        & {\color[HTML]{333333} —} & {\color[HTML]{333333} —} & {\color[HTML]{333333} —} & {\color[HTML]{333333} —} \\
\textit{lili\_8}                                & \textbf{14.73}                                                  & 15.31                                                           & 17.73                                                          & 17.15                                                          & 18.37                                                              & 57.29                                                          & 61.80        & 79.11        & {\color[HTML]{333333} —} & {\color[HTML]{333333} —} & {\color[HTML]{333333} —} & {\color[HTML]{333333} —} \\
\textit{utbm\_8}                                & 21.72                                                           & 22.05                                                           & 21.39                                                          & \textbf{20.82}                                                 & 21.16                                                              & 100.00                                                         & 65.29        & 84.76        & {\color[HTML]{333333} —} & {\color[HTML]{333333} —} & 37.44                    & 153.92                   \\
\textit{utbm\_9}                                & 28.26                                                           & 25.44                                                           & 21.41                                                          & 21.35                                                          & \textbf{17.46}                                                     & 91.05                                                          & 68.94        & 97.90        & {\color[HTML]{333333} —} & {\color[HTML]{333333} —} & 38.82                    & 154.06                   \\
\textit{utbm\_10}                               & 23.90                                                           & 22.48                                                           & 23.09                                                          & 20.74                                                          & \textbf{15.30}                                                     & 94.62                                                          & 66.10        & 97.29        & {\color[HTML]{333333} —} & {\color[HTML]{333333} —} & 33.61                    & 166.12                   \\
\textit{ulhk\_4}                                & 20.86                                                           & 20.14                                                           & 19.96                                                          & \textbf{20.04}                                                 & 29.35                                                              & 91.12                                                          & 52.40        & 74.80        & 39.50                    & 95.29                    & 34.72                    & 93.70                    \\
\textit{ulhk\_5}                                & 24.10                                                           & 23.90                                                           & 23.96                                                          & \textbf{23.75}                                                 & 28.70                                                              & 68.04                                                          & 53.56        & 47.68        & 25.68                    & 127.63                   & 28.01                    & 99.13                    \\
\textit{ulhk\_6}                                & 30.52                                                           & 31.56                                                           & 30.15                                                          & 29.25                                                          & 31.94                                                              & 92.38                                                          & 64.46        & 70.43        & \textbf{15.16}           & 164.36                   & 41.54                    & 199.96                   \\
\textit{nclt\_4}                                & 15.65                                                           & 15.72                                                           & 15.79                                                          & 15.75                                                          & 19.98                                                              & 69.09                                                          & 62.49        & 98.46        & \textbf{13.38}           & 184.03                   & 46.43                    & 188.40                   \\
\textit{nclt\_5}                                & 16.56                                                           & 16.60                                                           & 16.61                                                          & 16.58                                                          & \textbf{13.54}                                                     & 68.95                                                          & 67.64        & 83.34        & 19.09                    & 184.46                   & 47.83                    & 198.88                   \\
\textit{nclt\_6}                                & 15.92                                                           & 15.84                                                           & 15.83                                                          & 15.68                                                          & \textbf{14.72}                                                     & 66.64                                                          & 76.10        & 133.25       & ×                        & ×                        & 54.48                    & 195.31                   \\
\textit{nclt\_7}                                & 16.79                                                           & 16.87                                                           & 16.82                                                          & 16.63                                                          & \textbf{15.16}                                                     & 70.24                                                          & 67.65        & 81.69        & 29.50                    & 211.18                   & 56.94                    & 197.71                   \\
\textit{nclt\_8}                                & 14.29                                                           & 14.25                                                           & 14.32                                                          & 14.14                                                          & \textbf{7.94}                                                      & 57.03                                                          & 53.54        & 57.54        & 16.30                    & 163.09                   & 53.53                    & 144.95                   \\
\textit{nclt\_9}                                & 13.73                                                           & 13.65                                                           & 13.60                                                          & 13.64                                                          & \textbf{10.30}                                                     & 54.82                                                          & 42.84        & 68.86        & 12.79                    & 118.35                   & 46.12                    & 149.45                   \\
\textit{nclt\_10}                               & 21.85                                                           & 21.79                                                           & 21.78                                                          & 21.61                                                          & \textbf{20.62}                                                     & 89.65                                                          & 82.92        & 130.96       & 23.13                    & 324.62                   & 83.12                    & 252.68                   \\
\textit{liosam\_1}                              & 16.95                                                           & 14.77                                                           & 14.65                                                          & 16.19                                                          & 15.93                                                              & 60.60                                                          & 48.45        & 84.28        & \textbf{13.47}           & 135.39                   & 24.13                    & 179.44                   \\
\textit{liosam\_2}                              & \textbf{11.11}                                                  & 11.47                                                           & 11.52                                                          & 11.19                                                          & 19.68                                                              & 45.27                                                          & 42.58        & 99.01        & 13.09                    & 154.69                   & 20.71                    & 160.66                   \\
\textit{liosam\_3}                              & 19.38                                                           & 16.64                                                           & {12.00}                                                 & 13.01                                                          & 12.37                                                              & 44.26                                                          & 38.42        & 64.02        & \textbf{11.32}           & 124.35                   & 40.47                    & 117.25                   \\ \bottomrule
\end{tabular}
\end{table*}

\subsection{Processing Time Evaluation}
Table.\ref{tab:runtime_benchmark} shows the processing time of FAST-LIO2 with different configurations, LILI-OM, LIO-SAM, and LINS in all the sequences. The FAST-LIO2 is an integrated odometry and mapping architecture, where at each step the map is updated following immediately the odometry update. Therefore, the total time (``Total" in Table.\ref{tab:runtime_benchmark}) includes all possible procedures occurred in the odometry, including feature extraction if any (e.g., for the feature-based variant), motion compensation, $k$NN search, and state estimation, and mapping. It should be noted that the mapping includes point insertion, box-wise delete, and tree re-balancing. On the other hand, LILI-OM, LIO-SAM, and LINS all have separate odometry (including feature extraction, and rough pose estimation) and mapping (such as back-end fusion in LILI-OM~\cite{li2021towards}, incremental smoothing and mapping in LIO-SAM\cite{T2020liosam} and Map-refining in LINS\cite{qin2019lins}), whose average processing time per LiDAR scan are referred to as ``Odo." and ``Map." respectively in Table.~\ref{tab:runtime_benchmark}. The two processing time is summed up to compare with FAST-LIO2. 

From Table.~\ref{tab:runtime_benchmark}, we can see that the FAST-LIO2 consumes considerably less time than other LIO methods, being x8 faster than LILI-OM, x10 faster than LIO-SAM, and x6 faster than LINS. Even if only considering the processing time for odometry of other methods, FAST-LIO2 is still faster in most sequences except for four. The overall processing time of fast-LIO2, including both odometry and mapping, is almost the same as the odometry part of LIO-SAM, x3 faster than LILI-OM and over x2 faster than LINS. Comparing the different variants of FAST-LIO2, the processing time for different map sizes are very similar, meaning that the mapping and $k$NN search with our {\it ikd-Tree} is insensitive to map size. Furthermore, the feature-based variant and direct method FAST-LIO2 have roughly similar processing times. Although feature extraction takes additional processing time to extract the feature points, it leads to much fewer points (hence less time) for the subsequent $k$NN search and state estimation. On the other hand, the direct method saves the feature extraction time for points registration. Allowed by the superior computation efficiency of FAST-LIO2, we further implemented it with the default map size (1000 $m$, see \ref{sec:benchmark_accuracy}) on the Khadas VIM3 (ARM) embedded computer. The run time results show that FAST-LIO2 can also achieve 10 $Hz$ real-time performance that has not been demonstrated on an ARM-based platform by any prior work.

\section{Real-world Experiments}\label{sec:experiment}
\subsection{Platforms}

\begin{figure}[t]
    \begin{center}
        {\includegraphics[width=1\columnwidth]{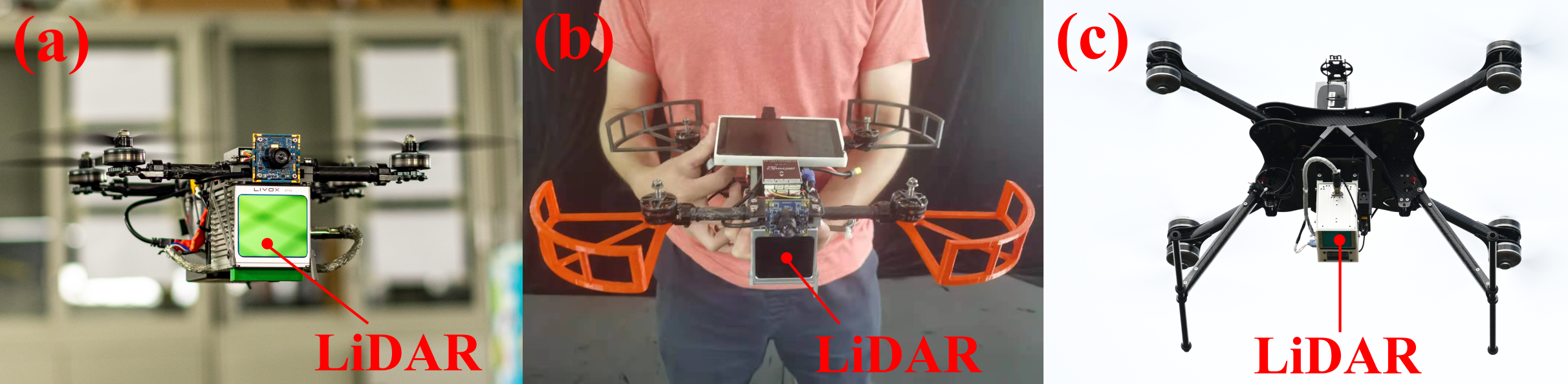}\vspace{-0.25cm}}
    \end{center}
    \caption{\label{fig:platforms}Three different platforms: (a) 280 $mm$ wheelbase small scale quadrotor UAV carrying a forward-looking Livox Avia LiDAR, (b) handheld platforms, (c) 750 $mm$ wheelbase quadrotor UAV carrying a down-facing Livox Avia LiDAR. All three platforms carry the same DJI Manifold-2C onboard computer. The video of real-world experiments is available at \url{https://youtu.be/2OvjGnxszf8}.}
\end{figure}

Besides the benchmark evaluation where the datasets are mainly collected on the ground, we also test our FAST-LIO2 in a variety of challenging data collected by other platforms (see Fig.~\ref{fig:platforms}), including a 280 $mm$ wheelbase quadrotor for the application of UAV navigation, a handheld platform for the application of mobile mapping, and a GPS-navigated 750 $mm$ wheelbase quadrotor UAV for the application of aerial mapping. The 280 $mm$ wheelbase quadrotor is used for indoor aggressive flight test, see section \ref{sec:uav_aggressive_test}, so that the LiDAR is installed face-forward. The 750 $mm$ wheelbase quadrotor UAV, developed by Ambit-Geospatial company\footnote[9]{\url{http://www.ambit-geospatial.com.hk}}, is used for the aerial scanning, see section \ref{sec:outdoor_aerial_test}, so that the LiDAR is facing down to the ground. In all platforms, we use a solid-state 3D LiDAR Livox Avia\footnote[10]{\url{https://www.livoxtech.com/de/avia}} which has a built-in IMU (model BMI088), a 70.4$^\circ$ (Horizontal) × 77.2$^\circ$ (Vertical) circular FoV, and an unconventional non-repetitive scan pattern that is different from the Livox Horizon or Velodyne LiDARs used previously in Section. \ref{sec:benchmark}. Since FAST-LIO2 does not extract features, it is naturally adaptable to this new LiDAR. In all the following experiments, FAST-LIO2 uses the default configurations (i.e., direct method with map size 1000 $m$). Unless stated otherwise, the scan rate is set at 100 $Hz$, and the computation platform is the DJI manifold 2-C used in the previous section.

\subsection{Private Dataset}
\subsubsection{Detail Evaluation of Processing Time}\label{sec:big_scene}
In order to validate the real-time performance of FAST-LIO2, we use the handheld platform to collect a sequence at 100 $Hz$ scan rate in a large-scale outdoor-indoor hybrid scene. The sensor returns to the starting position after traveling around $650m$. It should be noted that the LILI-OM also supports solid-state LiDAR, but it fails in this data since its feature extraction module produces too few features at the 100 $Hz$ scan rate. The map built by FAST-LIO2 in real-time is shown in Fig.~\ref{fig:large}, which shows small drift (i.e., 0.14 $m$) and good agreement with satellite maps.
\begin{figure}[t]
    \begin{center}
        {\includegraphics[width=1\columnwidth]{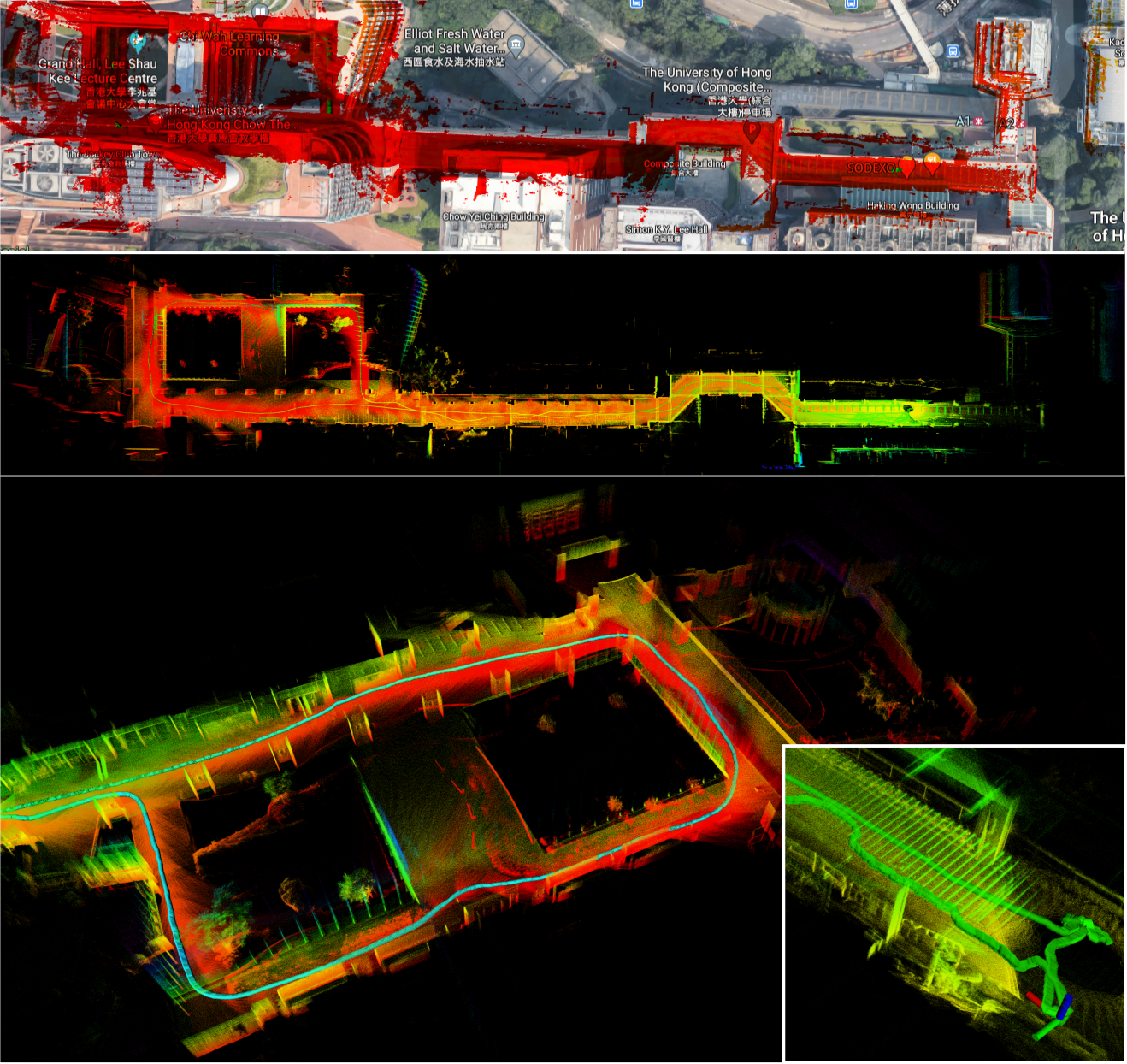}\vspace{-0.25cm}}
    \end{center}
    \caption{\label{fig:large}Large-scale scene experiment.}
\end{figure}

For the computation efficiency, we compare FAST-LIO2 with its predecessor FAST-LIO \cite{xu2020fastlio} on the Intel (Manifold 2-C) computer. For FAST-LIO2, we additionally test on the ARM (Khadas VIM3) onboard computer. The difference between these two methods is that FAST-LIO is a feature-based method, and it retrieves map points in the current FoV to build a new static k-d tree for $k$NN search at every step. The detailed time consumption of individual components for processing a scan is shown in Table.~\ref{tab:detail_runtime}. The preprocessing refers to data reception and formatting, which are identical for FAST-LIO and FAST-LIO2 and are below 0.1 $ms$. The feature extraction of FAST-LIO is 0.9 $ms$ per scan, which is saved by FAST-LIO2. The feature extraction leads to fewer point numbers than FAST-LIO2 (447 versus 756), hence less time spent in state estimation (0.99 $ms$ versus 1.66 $ms$). As a result, the overall odometry time of the two methods is nevertheless very close (1.92 $ms$ for FAST-LIO versus 1.69 $ms$ for FAST-LIO2). The difference between these two methods becomes drastic when looking at the mapping module, which includes map points retrieve and k-d tree building for FAST-LIO, and point insertion, box-wise delete due to map move and tree rebalancing for FAST-LIO2. As can be seen, the averaging mapping time per scan for FAST-LIO exceeds 10 $ms$ hence cannot be processed in real-time for this large scene. On the other hand, the mapping time for FAST-LIO2 is well below the sampling period. The overall time for FAST-LIO2 when processing 756 points per scan, including both odometry and mapping, is only 1.82 $ms$ for the Intel processor and 5.23 $ms$ for the ARM processor.
\begin{table}[t]
\footnotesize
\centering
\caption{Mean Time Consumption in Miliseconds by Individual Components when Processing A LiDAR Scan}
\label{tab:detail_runtime}
\begin{tabular}{@{}lrcrr@{}}
\toprule
  &  FAST-LIO && \multicolumn{2}{c}{FAST-LIO2}\\
\cmidrule{2-2} \cmidrule{4-5} 
                   & Intel && Intel& ARM \\
\midrule
Preprocessing         & 0.03 $ms$ && 0.03 $ms$ & 0.05 $ms$ \\
Feature extraction & 0.90 $ms$ && 0 $ms$ & 0 $ms$ \\
State estimation   & 0.99 $ms$ && 1.66 $ms$ & 4.75 $ms$ \\

Mapping            & 13.81 $ms$ && 0.13 $ms$ & 0.43 $ms$ \\

Total              & 15.83 $ms$ && 1.82 $ms$ & 5.23 $ms$ \\
Num. of points used & 447      && 756       & 756  \\
Num. of threads    & 4         && 4         & 2    \\
\bottomrule
\end{tabular}
\end{table}
\begin{figure}[t]
    \setlength\abovecaptionskip{-0.1\baselineskip}
    \centering
    \includegraphics[width=0.485\textwidth]{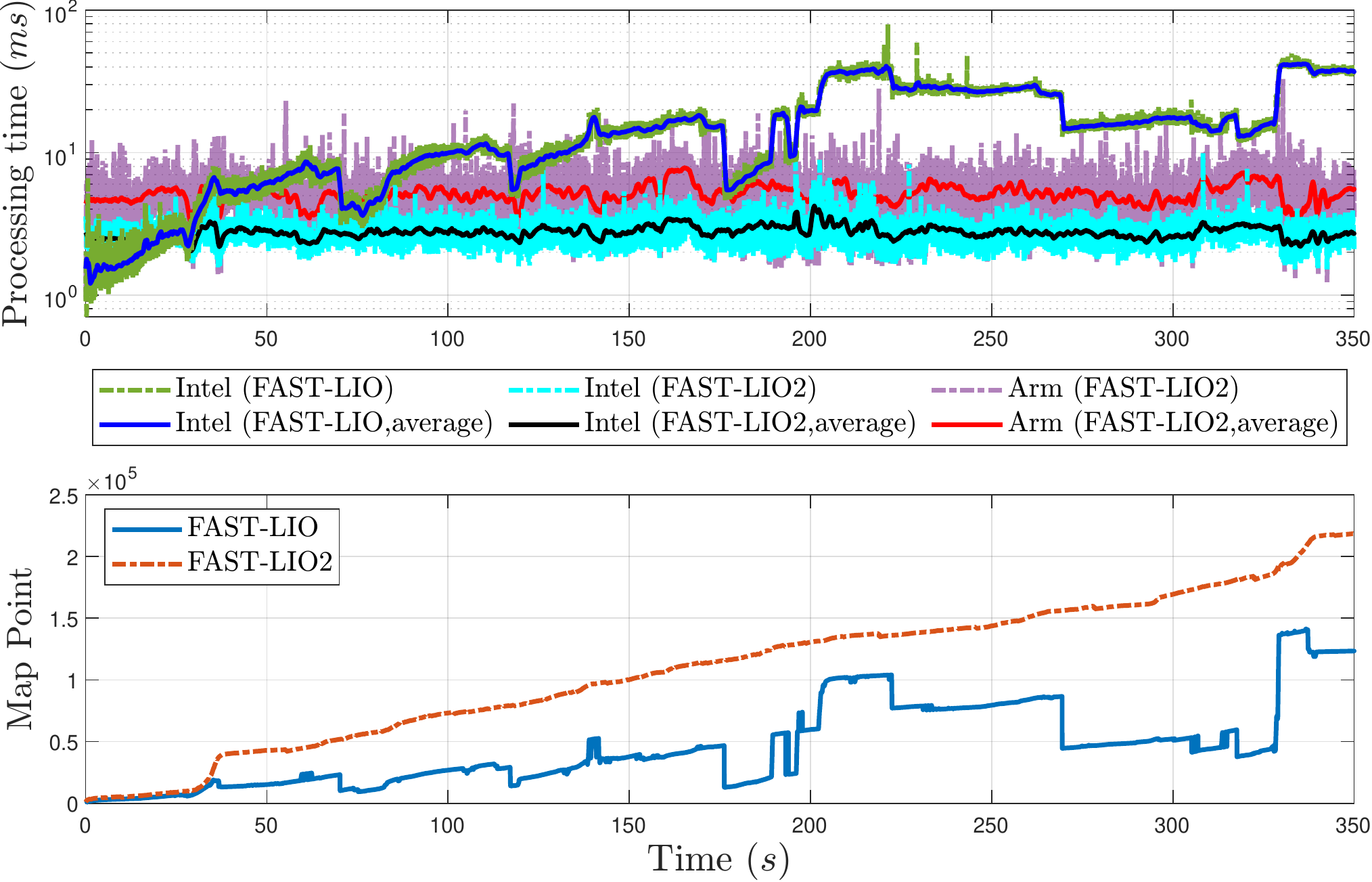}
    \caption{The processing time for each LiDAR scan of FAST-LIO and FAST-LIO2. }
    \label{fig:time_compare}
\end{figure}

The time consumption and the number of map points at each scan are shown in Fig. \ref{fig:time_compare}. As can be seen, the processing time for FAST-LIO2 running on the ARM processor occasionally exceeds the sampling period 10 $ms$, but this occurred very few and the average processing time is well below the sampling period. The occasional timeout usually does not affect a subsequent controller since the IMU propagated state estimate could be used during this short period. On the Intel processor, the processing time for FAST-LIO2 is always below the sampling period. On the other hand, the processing time for FAST-LIO quickly grows above the sampling period due to the growing number of map points. Notice that the considerably reduced processing time for FAST-LIO2 is achieved even at a much higher number of map points. Since FAST-LIO only retains map points within its current FoV, the number could drop if the LiDAR faces a new area containing few previously sampled map points. Even with fewer map points, the processing time for FAST-LIO is still much higher, as analyzed above. Moreover, since FAST-LIO builds a new k-d tree at every step, the building time has a time complexity $O(n\log n)$~\cite{bentley1975kdtree} where $n$ is the number of map points in the current FoV. This is why the processing time for FAST-LIO is almost linearly correlated to the map size. In contrast, the incremental updates of our {\it ikd-Tree} has a time complexity of $O(\log n)$, leading to a much slower increment in processing time over map size.

\subsubsection{Aggressive UAV Flight Experiment}\label{sec:uav_aggressive_test}
\begin{figure}[t]
    \begin{center}
        {\includegraphics[width=1.0\columnwidth]{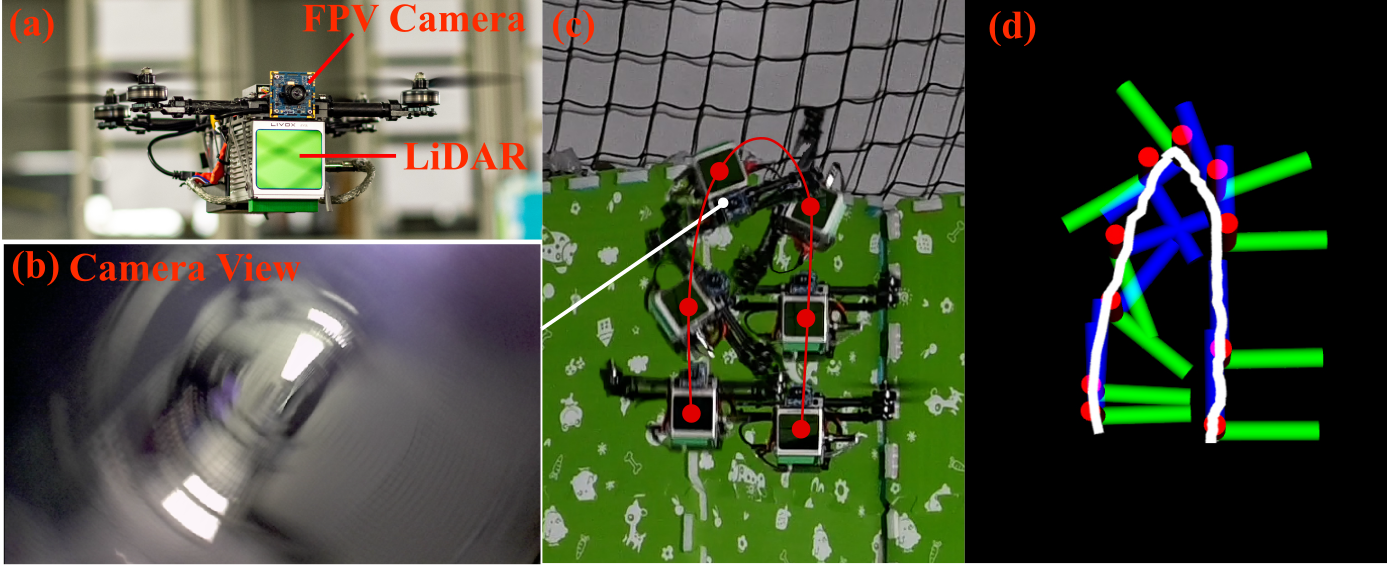}}
    \end{center}
    \vspace{-0.5cm}
    \caption{\label{fig:uav_fly} The flip experiment. (a) the small scale UAV; (b) the onboard camera showing first person view (FPV) images during the flip; (c) the third person view images of the UAV during the flip; (d) the estimated UAV pose with FAST-LIO2.}
\end{figure}
\begin{figure}[t]
    \begin{center}
        {\includegraphics[width=1.0\columnwidth]{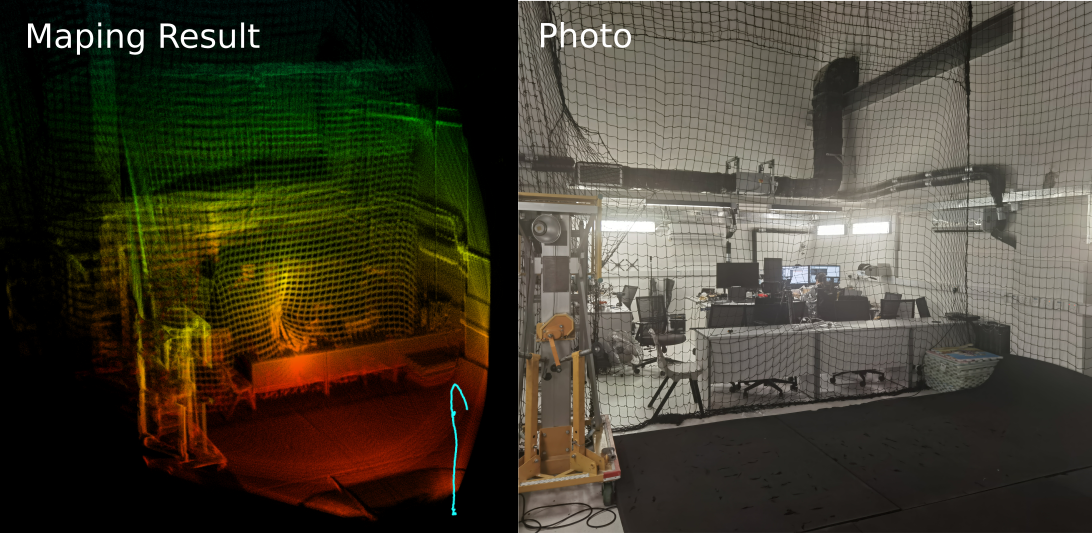}}
    \end{center}
    \vspace{-0.5cm}
    \caption{\label{fig:uav_indoor} The actual environment and the 3D map built by FAST-LIO2 during the flip.}
\end{figure}
\begin{figure}[t]
    \begin{center}
        {\includegraphics[width=1.0\columnwidth]{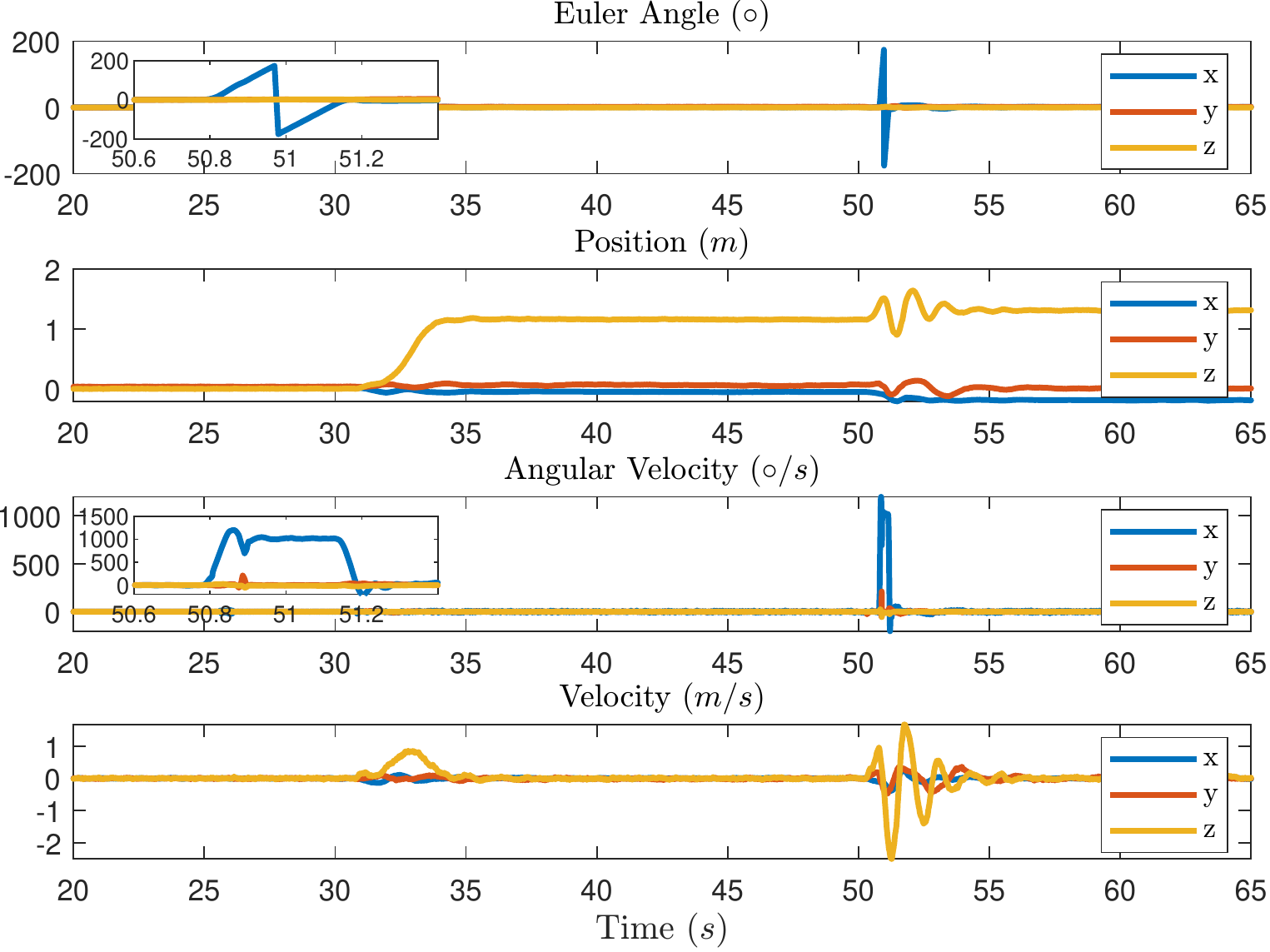}}
    \end{center}
    \vspace{-0.5cm}
    \caption{\label{fig:tumb_imu}The attitude, position, angular velocity and linear velocity in the UAV flip experiments.}
\end{figure}
In order to show the application of FAST-LIO2 in mobile robotic platforms, we deploy a small-scale quadrotor UAV carrying the Livox AVIA LiDAR sensor and conduct an aggressive flip experiment as shown in Fig.~\ref{fig:uav_fly}. In this experiment, the UAV first takes off from the ground and hovers at the height of 1.2 $m$ for a while, then it performs a quick flip, after which it returns to the hover flight under the control of an on-manifold model predictive controller \cite{lu2021model} that takes state feedback from the FAST-LIO2. The pose estimated by FAST-LIO2 is shown in Fig.~\ref{fig:uav_fly} (d), which agrees well with the actual UAV pose. The real-time mapping of the environment is shown in Fig.~\ref{fig:uav_indoor}. In addition, Fig.~\ref{fig:tumb_imu} shows the position, attitude, angular velocity, and linear velocity during the experiments. The average and maximum angular velocity during the flip reaches 912 $deg/s$ and 1198 $deg/s$, respectively (from 50.8 $s$ to 51.2 $s$). FAST-LIO2 takes only 2.01 $ms$ on average per scan, which suffices the real-time requirement of controllers. By providing high-accuracy odometry and a high-resolution 3D map of the environment at 100 $Hz$, FAST-LIO2 is very suitable for a robots' real-time control and obstacle avoidance. For example, our prior work \cite{kong2021avoiding} demonstrated the application of FAST-LIO2 on an autonomous UAV avoiding dynamic small objects (down to 9 $mm$) in complex indoor and outdoor environments.
\subsubsection{Fast Motion Handheld Experiment}\label{sec:fast_motion}
Here we test FAST-LIO2 in a challenging fast motion with large velocity and angular velocity. The sensor is held on hands while rushing back and forth on a footbridge (see Fig. \ref{fig:indoor_map}). Fig.~\ref{fig:indoor_imu} \begin{figure}[t]
    \begin{center}
        {\includegraphics[width=1.0\columnwidth]{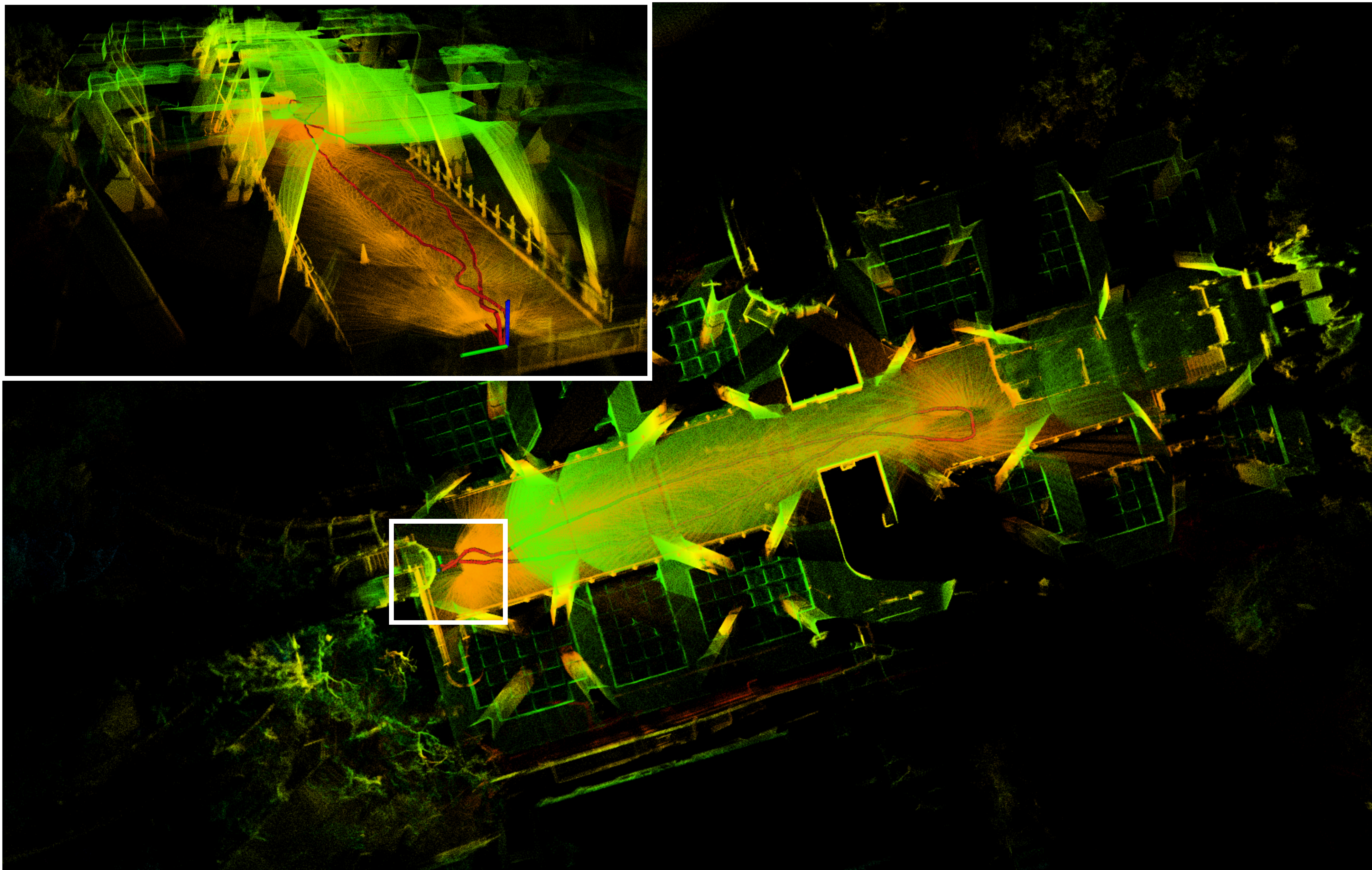}\vspace{-0.25cm}}
    \end{center}
    \vspace{-0.25cm}
    \caption{\label{fig:indoor_map}The mapping results of FAST-LIO2 in the fast motion handheld experiment.}
\end{figure}
\begin{figure}[t]
    \begin{center}
        {\includegraphics[width=1.0\columnwidth]{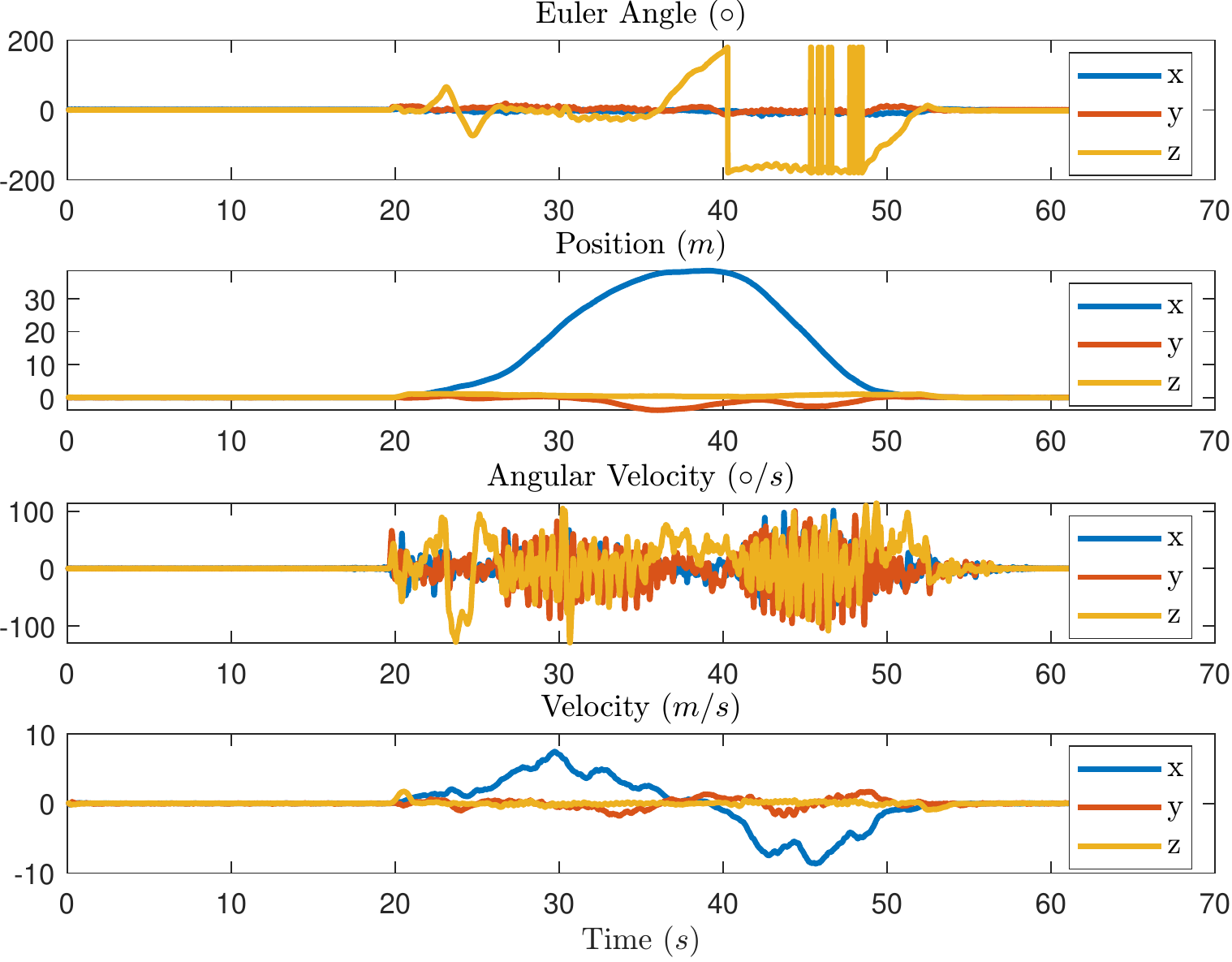}}
    \end{center}
    \vspace{-0.5cm}
    \caption{\label{fig:indoor_imu}The attitude, position, angular velocity and linear velocity in the fast motion handheld experiments.}
\end{figure} shows the attitude, position, angular velocity, and linear velocity in the fast motion handheld experiments. It is seen that the maximum velocity reaches 7 $m/s$ and angular velocity varies around $\pm$100 $deg/s$. In order to show the performance of FAST-LIO2, the experiment starts and ends at the same point. The end-to-end error in this experiment is less than 0.06 $m$ (see Fig.~\ref{fig:indoor_imu}) while the total trajectory length is 81 $m$.

\subsection{Outdoor Aerial Experiment}\label{sec:outdoor_aerial_test}
\begin{figure}[t]
    \begin{center}
        {\includegraphics[width=1\columnwidth]{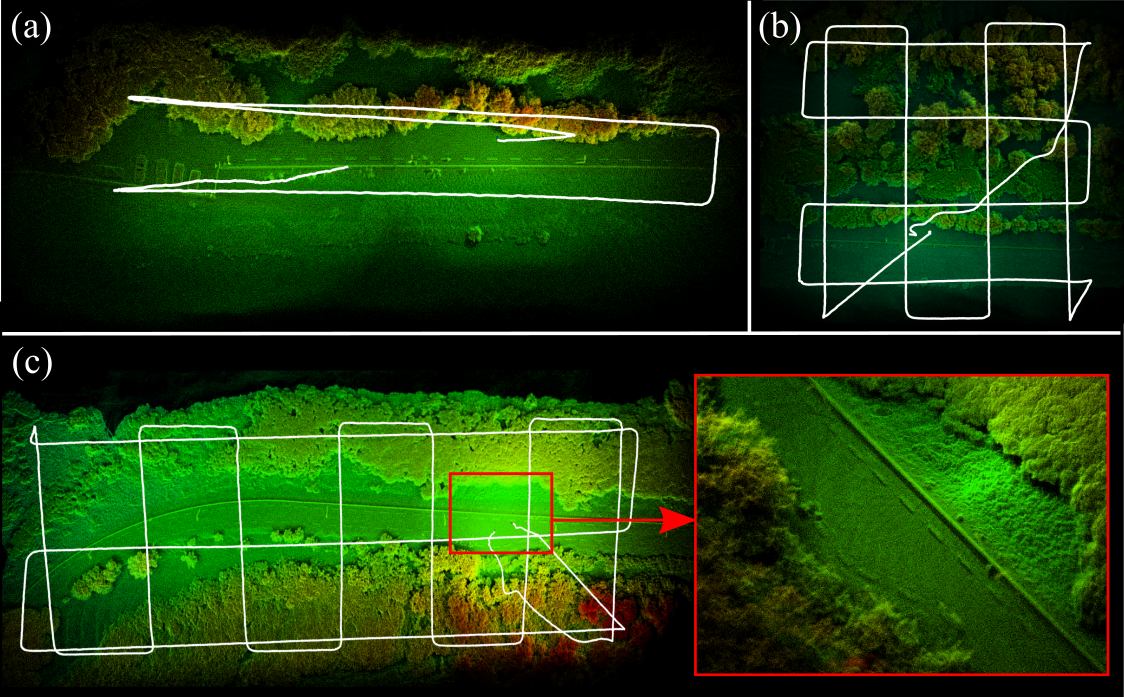}\vspace{-0.5cm}}
    \end{center}
    \caption{\label{fig:areial} Real-time mapping results with FAST-LIO2 for airborne mapping. The data is collected in the Hong Kong Wetland Park by a UAV with a down-facing Livox Avia LiDAR. The flight heights are 30 $m$ (a), 30 $m$ (b) and 50 $m$ (c).}
    \vspace{-0.1cm}
\end{figure}
One important application of 3D LiDARs is airborne mapping. In order to validate FAST-LIO2 for this possible application, an aerial experiment is conducted. A larger UAV carrying our LiDAR sensor is deployed. The UAV is equipped with GPS, IMU, and other flight avionics and can perform automatic waypoints following based on the onboard GPS/IMU navigation. Note that the UAV-equipped GPS and IMU are only used for the UAV navigation, but not for FAST-LIO2, which uses data only from the LiDAR sensor. The LiDAR scan rate is set to 10 $Hz$ in this experiment. A few flights are conducted in several locations in the Hong Kong Wetland Park at Nan Sang Wai, Hong Kong. The real-time mapping results are shown in Fig.~\ref{fig:areial}. It is seen that FAST-LIO2 works quite well in these vegetation environments. Many fine structures such as tree crowns, lane marks on the road, and road curbs can be clearly seen. Fig.~\ref{fig:areial} also shows the flight trajectories computed by FAST-LIO2. We have visually compared these trajectories with the trajectories estimated by the UAV onboard GPS/IMU navigation, and they show good agreement. Due to technical difficulties, the GPS trajectories are not available here for quantitative evaluation. Finally, the average processing time per scan for these three environments is 19.6 $ms$, 23.9 $ms$, and 23.7 $ms$, respectively. It should be noted that the LILI-OM fails in all these three data sequences because the extracted features are too few when facing the ground.

\section{Conclusion}\label{sec:conclusion}
This paper proposed FAST-LIO2, a direct and robust LIO framework significantly faster than the current state-of-the-art LIO algorithms while achieving highly competitive or better accuracy in various datasets. The gain in speed is due to removing the feature extraction module and the highly efficient mapping. A novel incremental k-d tree (\textit{ikd-Tree}) data structure, which supports dynamically point insertion, delete and parallel rebuilding, is developed and validated. A large amount of experiments in open datasets shows that the proposed \textit{ikd-Tree} can achieve the best overall performance among the state-of-the-art data structure for $k$NN search in LiDAR odometry. As a result of the mapping efficiency, the accuracy and the robustness in fast motion and sparse scenes are also increased by utilizing more points in the odometry. A further benefit of FAST-LIO2 is that it is naturally adaptable to different LiDARs due to the removal of feature extraction, which has to be carefully designed for different LiDARs according to their respective scanning pattern and density.

\section*{Acknowledgement}
This project is supported by DJI under the grant 200009538. The authors gratefully acknowledge DJI for the fund support, and Livox Technology for the equipment support during the whole work. The authors would like to thank Ambit-Geospatial for the helps in the outdoor aerial experiment. The authors also appreciate Zheng Liu, Guozheng Lu, and Fangcheng Zhu for the helpful discussions.

\appendix \label{appendix:sequence_list}
The detail information about all 37 sequences used in Section. \ref{sec:benchmark} are listed in Table. \ref{tab:append:detail_sequences}.
\begin{table}[t]
\footnotesize
\centering
\caption{Details of all the sequences for the Benchmark}
\label{tab:append:detail_sequences}
\begin{tabular}{@{}llrr@{}}
\toprule
\textit{}          & Name               & \multicolumn{1}{l}{\begin{tabular}[c]{@{}l@{}}Duration\\ ($min$:$sec$)\end{tabular}} & \multicolumn{1}{l}{\begin{tabular}[c]{@{}l@{}}Distance\\ ($km$)\end{tabular}} \\ \midrule
\textit{lili\_1}   & FR-IOSB-Tree       & 2:58                                     & 0.36                                \\
\textit{lili\_2}   & FR-IOSB-Long       & 6:00                                     & 1.16                                \\
\textit{lili\_3}   & FR-IOSB-Short      & 4:39                                     & 0.49                                \\
\textit{lili\_4}   & KA-URBAN-Campus-1  & 5:58                                     & 0.50                                \\
\textit{lili\_5}   & KA-URBAN-Campus-2  & 2:07                                     & 0.20                                \\
\textit{lili\_6}   & KA-URBAN-Schloss-1 & 10:37                                    & 0.65                                \\
\textit{lili\_7}   & KA-URBAN-Schloss-2 & 12:17                                    & 1.10                                \\
\textit{lili\_8}   & KA-URBAN-East      & 20:52                                    & 3.70                                \\
\textit{utbm\_1}   & 20180713           & 16:59                                    & 5.03                                \\
\textit{utbm\_2}   & 20180716           & 15:59                                    & 4.99                                \\
\textit{utbm\_3}   & 20180717           & 15:59                                    & 4.99                                \\
\textit{utbm\_4}   & 20180718           & 16:39                                    & 5.00                                \\
\textit{utbm\_5}   & 20180720           & 16:45                                    & 4.99                                \\
\textit{utbm\_6}   & 20190110           & 10:59                                    & 3.49                                \\
\textit{utbm\_7}   & 20190412           & 12:11                                    & 4.82                                \\
\textit{utbm\_8}   & 20180719           & 15:26                                    & 4.98                                \\
\textit{utbm\_9}   & 20190131           & 16:00                                    & 6.40                                \\
\textit{utbm\_10}  & 20190418           & 11:59                                    & 5.11                                \\
\textit{ulhk\_1}   & HK-Data20190316-1  & 2:55                                     & 0.23                                \\
\textit{ulhk\_2}   & HK-Data20190426-1  & 2:30                                     & 0.55                                \\
\textit{ulhk\_3}   & HK-Data20190317    & 5:18                                     & 0.62                                \\
\textit{ulhk\_4}   & HK-Data20190117    & 5:18                                     & 0.60                                \\
\textit{ulhk\_5}   & HK-Data20190316-2  & 6:05                                     & 0.66                                \\
\textit{ulhk\_6}   & HK-Data20190426-2  & 4:20                                     & 0.74                                \\
\textit{nclt\_1}   & 20120118           & 93:53                                    & 6.60                                \\
\textit{nclt\_2}   & 20120122           & 87:19                                    & 6.36                                \\
\textit{nclt\_3}   & 20120202           & 98:37                                    & 6.45                                \\
\textit{nclt\_4}   & 20120115           & 111:46                                   & 4.01                                \\
\textit{nclt\_5}   & 20120429           & 43:17                                    & 1.86                                \\
\textit{nclt\_6}   & 20120511           & 84:32                                    & 3.13                                \\
\textit{nclt\_7}   & 20120615           & 55:10                                    & 1.62                                \\
\textit{nclt\_8}   & 20121201           & 75:50                                    & 2.27                                \\
\textit{nclt\_9}   & 20130110           & 17:02                                    & 0.26                                \\
\textit{nclt\_10}  & 20130405           & 69:06                                    & 1.40                                \\
\textit{liosam\_1} & park               & 9:11                                     & 0.66                                \\
\textit{liosam\_2} & garden             & 5:58                                     & 0.46                                \\
\textit{liosam\_3} & campus             & 16:26                                    & 1.44                                \\ \bottomrule
\end{tabular}
\end{table}

\bibliography{paper}
\end{document}